\documentclass{article}

\PassOptionsToPackage{numbers}{natbib}

\usepackage[preprint]{neurips_2026}

\usepackage[utf8]{inputenc} 
\usepackage[T1]{fontenc}    
\usepackage{hyperref}       
\usepackage{url}            
\usepackage{booktabs}       
\usepackage{amsfonts}       
\usepackage{nicefrac}       
\usepackage{microtype}      
\usepackage{xcolor}         

\usepackage{amsmath}
\usepackage{amssymb}
\usepackage{mathtools}
\usepackage{amsthm}
\usepackage{mathabx}
\usepackage{bbm}
\usepackage{mathrsfs}
\usepackage{blkarray}
\usepackage{bm}
\usepackage{cancel}
\usepackage{xltabular}
\usepackage{abbreviation_definitions}

\usepackage{enumerate}
\usepackage{enumitem}

\usepackage{caption}
\usepackage{graphicx}
\usepackage{subcaption}
\usepackage{multirow}
\usepackage{wrapfig}

\usepackage{mwe}

\usepackage{tikz}
\usetikzlibrary{shapes,arrows,positioning,fit,decorations.pathreplacing,calc}

\usepackage{todonotes}
\usepackage{placeins}
\usepackage{algorithm}
\usepackage{algpseudocode}

\usepackage[most]{tcolorbox}
\usepackage{tabularx}
\usepackage{colortbl}
\usepackage{supertabular}

\definecolor{tracegray}{HTML}{F7F7F7}
\definecolor{traceblue}{HTML}{EEF5FF}
\definecolor{traceorange}{HTML}{FFF3D6}

\newcolumntype{Y}{>{\raggedright\arraybackslash}X}

\newtcolorbox{tracecasebox}[2][]{%
  enhanced,
  breakable,
  colback=white,
  colframe=black!20,
  colbacktitle=traceblue,
  coltitle=black,
  fonttitle=\bfseries,
  boxrule=0.4pt,
  arc=1.5pt,
  left=1.5mm,
  right=1.5mm,
  top=1mm,
  bottom=1mm,
  title={#2},
  #1
}

\newtcolorbox{traceexcerpt}[1][]{%
  enhanced,
  colback=tracegray,
  colframe=black!15,
  boxrule=0.3pt,
  arc=1pt,
  left=1.5mm,
  right=1.5mm,
  top=1mm,
  bottom=1mm,
  #1
}

\definecolor{nfpoPurple}{HTML}{9C1BA3}

\title{Multi-Step Likelihood-Ratio Correction for Reinforcement Learning with Verifiable Rewards}

\author{
Deokgyu Yoon$^{1}$\thanks{Equal contribution. Code is available at 
\url{https://github.com/oh-lab/NFPO}.} \quad
Hyungkyu Kang$^{2}$\footnotemark[1] \quad
Joongkyu Lee$^{1}$\footnotemark[1] \quad
Byeongchan Kim$^{1}$ \\
\textbf{Gyungin Shin}$^{2}$ \quad
\textbf{Sungrae Park}$^{2}$ \quad
\vspace{0.2cm}
\textbf{Min-hwan Oh}$^{1}$ \\
$^{1}$Seoul National University \qquad
\vspace{0.2cm}
$^{2}$Upstage \\
{\small\texttt{\{hhbdgy,jklee0717,bckim97,minoh\}@snu.ac.kr}} \\
{\small\texttt{\{hkang,noel,sungrae.park\}@upstage.ai}}
}

\begin{document}

\maketitle

\begin{figure}[h!]
    \centering
    \includegraphics[width=0.9\linewidth]{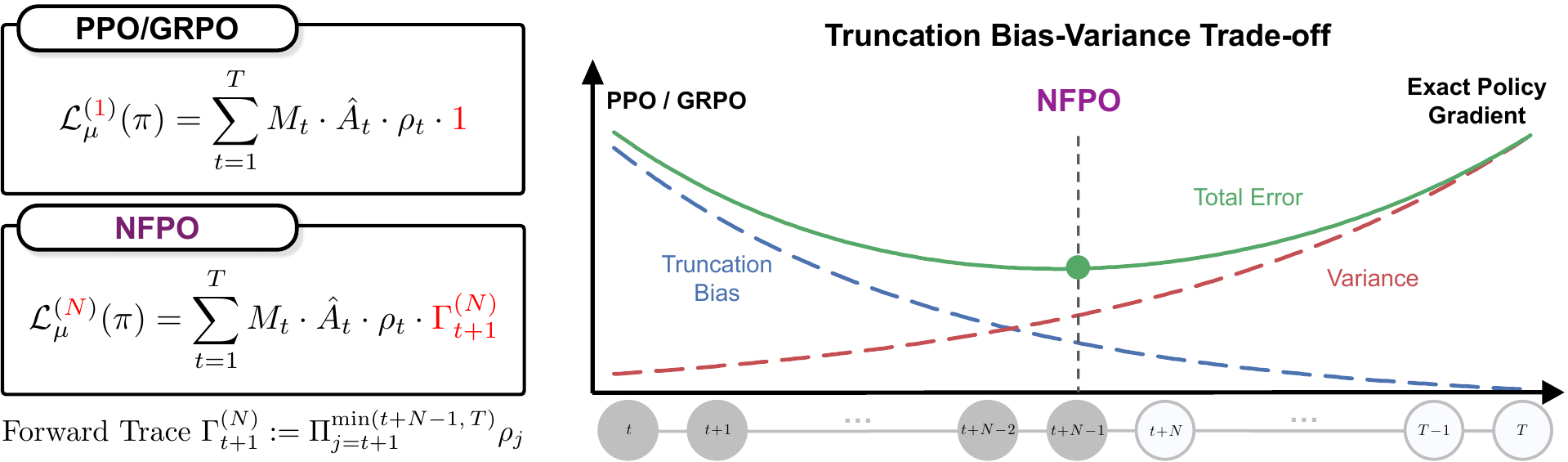}
    \caption{{\small 
    \textcolor{nfpoPurple}{\textbf{\AlgName{}}} corrects the bias of \textit{local} PPO/GRPO-style objectives by reweighting each token likelihood ratio 
    $\rho_t := \frac{\pi(y_t|s_t)}{\mu(y_t|s_t)}$ 
    with a \textit{forward trace} $\Gamma^{(N)}_{t+1}$.
    (Right) The trace horizon $N$ controls the bias--variance trade-off, interpolating between PPO/GRPO ($N=1$) and the \textit{exact} policy gradient objective ($N=T$).
    }}
    \label{fig:ablation_main}
\end{figure}
\begin{abstract}
Reinforcement learning with verifiable rewards (RLVR) plays a pivotal role in improving the reasoning ability of large language models.
However, widely used PPO surrogate objectives are fundamentally \textit{local}, as they rely on a local approximation of the exact policy gradient objective.
While this approximation improves stability by reducing the variance induced by importance sampling, it also introduces structural bias into the surrogate objective, which must be controlled through trust region mechanisms.
In this work, we introduce the $N$-step forward trace, which augments the PPO surrogate objective using the cumulative likelihood ratio of the next $N-1$ tokens.
Building on this idea, we propose $N$-Step Forward-Trace Policy Optimization (\AlgName{}), a practical RLVR algorithm that integrates the $N$-step forward trace into the masked policy gradient framework.
\AlgName{} provides a continuous bridge between the PPO surrogate objective and the exact policy gradient objective, offering a principled mechanism for controlling the bias–variance trade-off.
Our theoretical analysis shows that, with an appropriate choice of $N$, the proposed objective yields a tighter policy-improvement bound than the standard PPO surrogate.
Experiments on comprehensive reasoning benchmarks demonstrate that \AlgName{} consistently improves performance, supporting our theoretical findings.

\end{abstract}

\section{Introduction}
\label{sec:Introduction}
Reinforcement learning with verifiable rewards (RLVR) has become a central paradigm for improving the reasoning ability of large language models (LLMs)~\citep{shao2024deepseekmath,guo2025deepseek, yu2025dapo}. 
By optimizing model-generated responses against verifier-provided rewards, \emph{Proximal Policy Optimization} (PPO)-style policy optimization~\citep{schulman2017proximal}  has been widely adopted in recent RLVR methods, including Group Relative Policy Optimization (GRPO)~\citep{shao2024deepseekmath,guo2025deepseek}, yielding strong empirical gains on complex reasoning tasks.

Despite these successes, existing PPO-style surrogate objectives remain fundamentally \textit{local} in how they account for policy change.
Fix a prompt \(x\). 
Let \(y=(y_1,\ldots,y_T)\) denote a response, \(\mu\) the rollout policy, and \(\pi\) the target policy.
The PPO surrogate objective optimizes a sum of token-level surrogate terms, each using the current-token likelihood ratio \(\rho_t:=\frac{\pi(y_t | x,y_{<t})}{\mu(y_t | x,y_{<t})}\) together with a response-level reward \(R(y)\) or advantage.
This locality makes training simple and stable, but it omits the effect of the updated policy on the probability of the remaining generated tokens.

This limitation becomes explicit through the exact policy-improvement identity (Lemma~\ref{lemma:pdl}).
For the expected reward \(\Jcal(\pi):=\EE_{y\sim\pi}[R(y)]\), the following holds:
\[
    \Jcal(\pi)-\Jcal(\mu)
    =
    \EE_{y\sim\mu}
    \left[
        R(y)
        \sum_{t=1}^{T}
        \bigl(\rho_t-1\bigr)
        \Gamma_{t+1}
    \right],
    \qquad
    \text{where}\quad
    \Gamma_{t+1}
    :=
    \prod_{j=t+1}^{T}\rho_j.
\]
Here, \(\Gamma_{t+1}\) is the forward likelihood-ratio correction for the tokens generated after step \(t\).
Throughout the paper, we call \(\Gamma_{t+1}\) the \textit{full forward
trace}.\footnote{
This terminology is descriptive and should not be confused with \textit{eligibility traces} in
temporal-difference learning~\citep{sutton2018reinforcement}. Eligibility traces are additive,
recursively decayed variables for backward credit assignment, whereas our forward trace is a multiplicative likelihood-ratio correction over the future tokens
\(y_{t+1:T}\).
}
The term ``trace'' emphasizes that this factor follows the generated trajectory forward from token
\(t+1\) to the end of the response, multiplying the likelihood-ratio corrections along the way.

The PPO surrogate objective corresponds to the case \(\Gamma_{t+1} = 1\), yielding a biased approximation of the exact policy gradient objective. 
In particular, the omitted forward correction is governed by the deviation between the future trajectory distributions induced by \(\mu\) and \(\pi\): conditional on \(s_{t+1}\), \(\EE_{\mu}[|\Gamma_{t+1}-1|\mid s_{t+1}]\) is controlled by the total variation distance between these future distributions. 
Therefore, under the uniform trust-region condition \(\sup_s D_{\mathrm{TV}}(\mu(\cdot | s),\pi(\cdot | s))\le \delta\), the resulting approximation error is bounded by \(\BigO(T^2\delta^2)\) (Proposition~\ref{prop:policy_improvement_bound}). 
Thus, the trust region controls not only optimization stability, but also the structural bias induced by localizing the exact objective. 
In practice, PPO-based RLVR methods approximate this trust-region control through ratio-based clipping~\citep{schulman2017proximal,shao2024deepseekmath,guo2025deepseek} or token-level update masking~\citep{team2025every,qi2026rethinking}.

However, trust-region mechanisms only indirectly compensate for the omitted full forward trace; they do not recover the missing forward likelihood-ratio correction itself. 
This leaves a structural bias--variance dilemma in policy improvement: the local surrogate is \textit{stable} but \textit{biased}, whereas the full forward trace is \textit{unbiased} but can yield highly \textit{unstable empirical estimates}, especially in long-horizon generation. 
Indeed, empirical optimization relies on a finite batch of rollouts, and the full trace \(\Gamma_{t+1}\) can vary sharply across samples because it multiplies many token-level likelihood ratios. 
Consequently, a small number of trajectories with large forward traces can dominate the empirical gradient, leading to high variance and unstable updates.

This motivates an intermediate objective that recovers part of the missing forward trace while avoiding the empirical instability of the full trace.
To address this dilemma, we introduce \textbf{\(\mathbf{N}\)-step forward trace}.
Rather than either discarding the forward trace correction entirely or using the full correction, we retain only the next \(N-1\) future likelihood ratios after each token, i.e., 
\(\Gamma_{t+1}^{(N)}:=\prod_{j=t+1}^{\min\{t+N-1,T\}}\rho_j\).
This construction interpolates between the PPO surrogate and the exact policy gradient objective: \(N=1\) gives \(\Gamma_{t+1}^{(1)}=1\), while \(N=T\) recovers \(\Gamma_{t+1}^{(T)}=\Gamma_{t+1}\).
Our main theorem (Theorem~\ref{thm:n_step_policy_improvement}) shows that increasing \(N\) reduces the \textit{truncation bias} at the rate \(\BigO((T-N)^2\delta^2)\), but increases the \textit{variance} exponentially in \(N\), 
since longer forward traces involve products of more likelihood ratios.
Thus, the trace horizon \(N\) provides a principled knob for controlling the bias--variance trade-off in RLVR.

We further implement \AlgName{} within a \textit{masked} policy gradient framework~\citep{team2025every, qi2026rethinking}, and demonstrate through extensive reasoning benchmarks that \AlgName{} consistently outperforms or matches both GRPO and DPPO.
Beyond performance gains, we conduct a detailed mechanistic analysis of the forward trace.
Our analysis reveals that \AlgName{} selectively concentrates forward trace correction on reasoning-related tokens, thereby improving the learning of reasoning behaviors.
In addition, the forward trace acts as a semantic \textit{low-pass filter} by aggregating learning signals from future tokens, which stabilizes the noisy token-level likelihood ratios that arise in local surrogate objectives.
Finally, through experiments with varying trace horizons $N$, we empirically validate the bias--variance trade-off predicted by our main theorem.
Together, these analyses complement our theoretical results by explaining how forward trace correction alters the underlying training dynamics.

Our main contributions are summarized as follows:
\begin{itemize}
    \item 
    \textbf{\(\mathbf{N}\)-step bias--variance trade-off:}
    We revisit the standard PPO surrogate objective, and show that it discards the full forward trace in the exact policy-improvement identity, thereby inducing a structural bias (Proposition~\ref{prop:policy_improvement_bound}).
    We then prove that an \(N\)-step forward trace reduces \textit{truncation bias} at the cost of higher \textit{variance}, implying that an appropriately chosen trace horizon \(N\) can yield a tighter policy-improvement lower bound than either the local surrogate (e.g., PPO) or the exact policy gradient (Theorem~\ref{thm:n_step_policy_improvement}).

    \item
    \textbf{\underline{$\mathbf{N}$}-Step \underline{F}orward-Trace \underline{P}olicy \underline{O}ptimization (\AlgName):}
    Guided by the theory, we propose a practical policy-optimization algorithm that uses a truncated forward trace with an appropriate trace horizon \(N\) (Section~\ref{sec:alg}).
    The method bridges the stable but biased local PPO surrogate and the unbiased but high-variance exact policy gradient objective.
    On complex reasoning benchmarks, \AlgName{} consistently outperforms baselines (Section~\ref{sec:main_results}). 
    
    \item
    \textbf{Learning-dynamics analysis:}
    Our analyses show that forward trace amplifies token-level correction signals, concentrates updates on reasoning tokens, and reduces token-level learning noise.
    Experiments varying the trace horizon \(N\) further highlight the importance of choosing an appropriate \(N\), supporting the bias--variance trade-off predicted by our theory (Section~\ref{sec:analysis}).
    
\end{itemize}

\section{Preliminaries}
\label{sec:preliminaries}
In this section, we introduce the basic setup and notation for policy optimization. 
Given a prompt dataset \(\Dcal\), a prompt \(x \sim \Dcal\) is sampled, and a large language model generates a response \(y=(y_1,\dots,y_T)\) autoregressively according to the policy
$\pi(y | x)=\prod_{t=1}^T \pi(y_t | x, y_{<t}).$
A verifier then assigns a scalar reward \(R(x,y)\in[0,1]\) to each prompt-response pair \((x,y)\).
To simplify notation, throughout the paper we condition on a fixed prompt \(x\), suppress the dependence on \(x\), and write
\[
    \Jcal(\pi| x)
    :=
    \EE_{y\sim\pi(\cdot | x)}
    [R(x,y)]
    \equiv
    \Jcal(\pi)
    :=
    \EE_{y\sim\pi}
    [R(y)].
\]
Let \(\mu\) denote the behavior policy (for rollout) and \(\pi\) the target policy (to be optimized).
We define the token-level state at \(t\) as \(s_t := (x,y_{<t})\).
Then the following performance-difference identity holds:
\begin{lemma}[Performance difference lemma]
\label{lemma:pdl}
For any two policies \(\pi\) and \(\mu\), we have
\begin{equation}
\label{eq:pdl_exact}
    \Jcal(\pi)-\Jcal(\mu)
    =
    \EE_{y\sim\mu}
    \left[
    R(y)\sum_{t=1}^T (\rho_t -1)
    \Gamma_{t+1}
    \right],
    \,\,\,
    \text{where}\,\,
    \Gamma_{t+1} :=\! 
    \prod_{j=t+1}^T \frac{\pi(y_j | s_j)}{\mu(y_j | s_j)}
    =
    \!\prod_{j=t+1}^T \rho_j.
\end{equation}
\end{lemma}
The proof is deferred to Appendix~\ref{appsubsec:proof_lemma:pdl}.
Lemma~\ref{lemma:pdl} shows that the policy improvement can be decomposed into token-level
policy deviations \(\rho_t-1\), with each deviation weighted by a future likelihood-ratio correction, which we call the \textit{full forward trace},
$\Gamma_{t+1} = \prod_{j=t+1}^T \frac{\pi(y_j\mid s_j)}{\mu(y_j\mid s_j)}$.
Conditional on the prefix up to
token \(t\), this factor is the likelihood ratio of the remaining tokens \(y_{t+1:T}\) under the candidate
policy \(\pi\) relative to the behavior policy \(\mu\).

\textbf{Trust region for RLVR. }\,
Though the objective in~\eqref{eq:pdl_exact} is exact, directly optimizing it is
unstable in long-horizon language generation. 
The full forward trace $\Gamma_{t+1}(\theta)$
compounds token-level policy mismatch over the future tokens. 
Thus, a few high-ratio trajectories can dominate the update, leading to high-variance and
unstable gradients.
To avoid this instability, we typically optimize a \textit{local, $\Gamma_{t+1}$-free} surrogate objective under a trust-region constraint.
\begin{proposition}[Policy improvement]
\label{prop:policy_improvement_bound}
Suppose \(|R(y)|\le \xi\) for all \(y\) and $\left|\rho_t-1\right|
    \le \epsilon$
for all $t$, for some \(\epsilon>0\).
Define the population local surrogate as 
$
\Lcal_{\mu}(\pi) :=
    \EE_{y\sim\mu}
    \big[
        R(y)\sum_{t=1}^T
        \bigl(\rho_t-1\bigr)
    \big]$.
The empirical surrogate is then defined as 
    $\widehat{\Lcal}_{\mu}(\pi)
    :=
    \frac1G\sum_{i=1}^{G}
    R(y^{(i)})
    \sum_{t=1}^{T}
    \bigl(\rho_t^{(i)}-1\bigr)$.
Then, with probability at least \(1-\alpha\), we have
\begin{align*}
    \Jcal(\pi)-\Jcal(\mu)
    &\ge
    \Lcal_{\mu}(\pi)
    -
    2\xi T(T-1)
    \Bigl(D_{\mathrm{TV}}^{\max}(\mu,\pi)\Bigr)^2 
    \\
    &\ge
    \widehat{\Lcal}_{\mu}(\pi)
    -
    2\xi T(T-1)
    \Bigl(D_{\mathrm{TV}}^{\max}(\mu,\pi)\Bigr)^2
    -
    \xi T\epsilon
    \sqrt{
        \frac{2\log(1/\alpha)}{G}} ,
    \numberthis 
    \label{eq:N=1_lower}
\end{align*}
where
    $D_{\mathrm{TV}}^{\max}(\mu,\pi)
    :=
    \sup_{s}
    D_{\mathrm{TV}}\!\left(
        \mu(\cdot | s),\pi(\cdot | s)
    \right)$.
\end{proposition}
The proof is deferred to Appendix~\ref{app_subsec:policy_improve}.
Proposition~\ref{prop:policy_improvement_bound} gives a lower bound on the one-step policy improvement in terms of the empirical local surrogate, a total-variation penalty, and a variance term.
The bound implies that increasing \(\widehat{\Lcal}_{\mu}(\pi)\) leads to policy improvement only when the target policy remains sufficiently close to the rollout policy \(\mu\), i.e., when \(D_{\mathrm{TV}}\) is small.
This provides a direct justification for trust-region mechanisms in conventional RLVR: they control the error incurred by replacing the exact policy gradient objective in~\eqref{eq:pdl_exact} with the local surrogate.
We therefore optimize:
\begin{align}
    \max_{\pi}\,
    \widehat{\Lcal}_{\mu}(\pi)
    =
    \frac1G\sum_{i=1}^{G}
    R(y^{(i)})
    \sum_{t=1}^{T}
    \bigl(\rho_t^{(i)}-1\bigr),
    \qquad
    \mathrm{s.t.}
    \quad
    D_{\mathrm{TV}}^{\max}(\mu,\pi)\le \delta, 
    \label{eq:trpo}
\end{align}
By Pinsker's inequality, the TV constraint can also be enforced via a KL-divergence constraint.

\textbf{Proximal Policy Optimization (PPO). }\,
Although the constrained objective in~\eqref{eq:trpo} has strong theoretical motivation, solving it requires second-order optimization, which is computationally expensive and difficult to scale.
PPO~\citep{schulman2017proximal} was introduced as a practical alternative that preserves the stabilizing effect of trust-region methods through a clipping operation:
\begin{align}
    \widehat{\Lcal}_{\mu}^{\mathrm{PPO}}(\pi)
    :=
    \frac{1}{G}\sum_{i=1}^{G}
    \sum_{t=1}^{T}
    \min\left(
        \rho_t^{(i)}\widehat A_t^{(i)},\,
        \operatorname{clip}\!\left(
            \rho_t^{(i)},
            1-\epsilon,
            1+\epsilon
        \right)\widehat A_t^{(i)}
    \right),
    \label{eq:ppo}
\end{align}
where \(\widehat{A}_t^{(i)}\) denotes the estimated advantage at token position \(t\) for the \(i\)-th response.
In LLM fine-tuning, the advantage is typically constructed from the centered group reward, e.g.,
\(\widehat A_t^{(i)}=\widehat A^{(i)}:=R(y^{(i)})-\frac{1}{G}\sum_{j=1}^{G}R(y^{(j)})\) for all \(t\)~\citep{shao2024deepseekmath,liu2025understanding}.
The clipping operation can be viewed as a \textit{sample-level approximation} to a
trust-region constraint.
Indeed, for each token state \(s_t\), \(D_{\mathrm{TV}}(\mu(\cdot | s_t),\pi(\cdot | s_t))=\frac{1}{2}\EE_{y_t\sim\mu(\cdot | s_t)}[|\rho_t-1|]\).
Thus, clipping \(\rho_t\) around \(1\) approximately limits the local policy shift from \(\mu\) to \(\pi\), providing a practical first-order substitute for the explicit trust-region constraint.

\textbf{Masked Policy Gradient (MPG). }\,
Recently, a growing body of work has emphasized that a key factor for stable RL training of LLMs is the control of training–inference mismatch~\citep{chen2025minimax,qi2025defeating,team2025every,zheng2025stabilizing,qi2026rethinking}.
In particular, masked policy gradient methods—which exclude a small subset of tokens exhibiting large discrepancies in log-probability between the rollout and training policies—have emerged as a general framework for addressing this issue.
The objective of MPG is
\begin{equation}
    \widehat{\Lcal}^{\rm MPG}_{\mu}(\pi)
    :=
    \frac{1}{G}\sum_{i=1}^{G}
    \sum_{t=1}^{T}
    M_t^{(i)}\cdot \rho_t^{(i)} \cdot \widehat A_t^{(i)},
    \label{eq:MPG}
\end{equation}
where \(M_t^{(i)}\in\{0,1\}\) is a token-level mask.
Several existing methods instantiate this template through different choices of \(M_t^{(i)}\).
For example, GRPO~\citep{shao2024deepseekmath} uses the asymmetric ratio-based mask \(M_t^{(i)}\!:=\!\mathbf{1}\{\epsilon_\mathrm{low}\!<\!\rho_t^{(i)}\!< \! \epsilon_\mathrm{high}\ \text{or}\ \widehat A_t^{(i)}(\rho_t^{(i)}-1)\le 0\}\), while DPPO~\citep{qi2026rethinking} replaces this with a divergence-based token mask, e.g., \(M_t^{(i)}:=\mathbf{1}\{D_{\mathrm{TV}}(\mu(\cdot|s_t^{(i)}),\pi(\cdot|s_t^{(i)}))\le\delta\ \text{or}\ \widehat A_t^{(i)}(\rho_t^{(i)}-1)\le 0\}\).
Thus, both methods induce a token-level trust-region effect by retaining update terms only when the policy shift is controlled or when the update moves in the unclipped direction.

\section{$N$-Step Forward-Trace Policy Optimization}
\label{sec:alg}
Lemma~\ref{lemma:pdl} shows that the full forward trace
\(\Gamma_{t+1}\) gives an exact, \textit{unbiased} expression for policy improvement.
In long-horizon generation, however, this full product can be \textit{unstable} and induce \textit{high variance}, since even modest token-level policy mismatches may compound across the remaining tokens.
At the other extreme, ignoring the forward trace entirely gives the \textit{biased} local surrogate underlying
standard RLVR objectives such as PPO/GRPO in~\eqref{eq:ppo},
where stability is enforced only through ratio clipping or token-level divergence masking.
Motivated by this bias--variance trade-off, we introduce an \textbf{\(\mathbf{N}\)-step forward trace} that retains only a finite window of future likelihood ratios, interpolating between the local RLVR surrogate and the exact policy gradient objective in~\eqref{eq:pdl_exact}.

For \(N\in\{1,\ldots,T\}\), define the $N$-step forward trace
    $\Gamma_{t+1}^{(N)}
    :=
    \prod_{j=t+1}^{\min\{t+N-1,T\}}
    \rho_j$.
This correction keeps only the next \(N-1\) future likelihood ratios after token \(t\). 
Using this truncated correction, we define the (population) $N$-step surrogate objective as
\begin{equation} \label{eq:N step surrogate loss}
    \Lcal_{\mu}^{(N)}(\pi)
    :=
    \EE_{y\sim\mu}
    \left[
        R(y)\sum_{t=1}^{T}
        \bigl(\rho_t-1\bigr)
        \Gamma_{t+1}^{(N)}
    \right].
\end{equation}
When \(N=1\), the product is empty, so \(\Gamma_{t+1}^{(1)}=1\). 
Therefore, the \(N\)-step surrogate reduces to the local surrogate, i.e.,
    $\Lcal_{\mu}^{(1)}(\pi)
    =
    \Lcal_{\mu}(\pi).$
At the other extreme, when \(N=T\), the correction includes all future tokens after token \(t\), so
\(\Gamma_{t+1}^{(T)}=\Gamma_{t+1}\). 
This recovers the exact policy-improvement
objective in~\eqref{eq:pdl_exact}.
While larger $N$ may appear preferable since $N=T$ recovers the exact policy gradient objective, this observation applies only to the \emph{population} objective.
In practice, optimization is performed using an empirical $N$-step surrogate computed from $G$ sampled responses, namely
$\widehat{\Lcal}_{\mu}^{(N)}(\pi)
    :=
    \frac1G\sum_{i=1}^{G}
    R(y^{(i)})
    \sum_{t=1}^{T}
    \bigl(\rho_t^{(i)}-1\bigr)
    \Gamma_{t+1}^{(N,i)}.$
As $N$ increases, the variance of this estimator also grows, so the practical behavior is governed by a trade-off between bias and variance.

\subsection{Theoretical Analysis}
\label{sec:theoretical analysis}

The intuition on the bias--variance trade-off of the $N$-step surrogate is formalized in the following theorem, where the proof is provided in Appendix~\ref{app_sec:proof_main}.

\begin{theorem}[\(N\)-step policy improvement]
\label{thm:n_step_policy_improvement}
Suppose \(|R(y)|\le \xi\) for all \(y\) and $\left|\rho_t-1\right|
    \le \epsilon$
for all $t$, for some \(\epsilon>0\).
Then, for any fixed \(\pi\), with probability at least \(1-\alpha\), we have
\begin{align*}
    \Jcal(\pi)-\Jcal(\mu)
    &\ge
    \Lcal_{\mu}^{(N)}(\pi)
    -
    2\xi (T-N)(T-N+1)
    \Bigl(D_{\mathrm{TV}}^{\max}(\mu,\pi)\Bigr)^2
    \\
    &\ge
    \widehat{\Lcal}_{\mu}^{(N)}(\pi)
    -
    \underbrace{2\xi (T-N)(T-N+1)
    \Bigl(D_{\mathrm{TV}}^{\max}(\mu,\pi)\Bigr)^2}_{\text{truncation bias: decreasing in }N}
    -
    \underbrace{B_{N}
    \sqrt{\frac{2\log(1/\alpha)}{G}}\,,}_{\text{variance: increasing in}\, N}
\end{align*}
where 
\(D_{\mathrm{TV}}^{\max}(\mu,\pi)
:=
\sup_s D_{\mathrm{TV}}(\mu(\cdot| s),\pi(\cdot| s))\)
and
$B_N
    :=
    \xi\epsilon
    \sum_{t=1}^{T}
    (1+\epsilon)^{\min\{N-1,T-t\}} $.
\end{theorem}
%

\textbf{Bias--variance trade-off. }\,
At the population level, the \(N\)-step surrogate \(\Lcal_{\mu}^{(N)}(\pi)\) lower bounds the true policy improvement up to a truncation-bias term. 
This bias decreases with \(N\), since a longer forward trace retains more future likelihood-ratio
corrections from the exact policy gradient objective in~\eqref{eq:pdl_exact}.
At the finite-sample level, replacing the population surrogate by its empirical counterpart introduces an additional variance penalty of order
$\BigO\big( (1+\epsilon)^{N-1} \big)$. 
Because larger \(N\) involves products of more likelihood ratios, this variance term can grow with \(N\). 
Thus, small \(N\) gives a stable but more biased objective, whereas large \(N\) gives a less biased but potentially less stable objective. 
Selecting an appropriate intermediate \(N\) is therefore essential for achieving consistent and stable policy improvement in practice.

\textbf{Bridging local surrogates and exact policy gradient objective. }\,
Theorem~\ref{thm:n_step_policy_improvement} makes this interpolation explicit by identifying the two endpoints of the \(N\)-step surrogate.
When \(N=1\), we have \(\Gamma_{t+1}^{(1)}=1\), so the resulting bound reduces to the local surrogate bound in~\eqref{eq:N=1_lower}, which underlies standard local surrogate objectives such as PPO/GRPO in~\eqref{eq:ppo}.
When \(N=T\), the correction includes all future tokens after token \(t\), so \(\Gamma_{t+1}^{(T)}=\Gamma_{t+1}\).
In this case, the truncation bias vanishes and the bound recovers the exact policy gradient objective in~\eqref{eq:pdl_exact}.
Thus, the proposed \(N\)-step formulation provides a continuous bridge between local policy optimization and exact policy improvement.

\textbf{Why choosing the right trace horizon matters. }\,
Theorem~\ref{thm:n_step_policy_improvement} suggests that neither endpoint should be used by default. 
The trace horizon \(N\) is a central design parameter: it determines how much future correction is incorporated into the objective, and therefore directly affects the balance between bias reduction and optimization stability. 
Choosing \(N\) too small may discard useful future correction and leave substantial truncation bias; choosing \(N\) too large may introduce unstable long-range likelihood-ratio products and inflate the finite-sample penalty. 
Thus, the best practical choice is often an intermediate trace horizon rather than either endpoint \(N=1\) or \(N=T\).

We therefore treat \(N\) as an important hyperparameter.
Larger values of \(N\) are preferable when the updated policy remains close to the behavior policy (i.e., small \(D_{\mathrm{TV}}^{\max}\)) and when token-level likelihood ratios are stable (i.e., small \(\epsilon\)).
Smaller values of \(N\) are preferable when the update approaches the trust-region boundary (i.e., large \(D_{\mathrm{TV}}^{\max}\)) or future-token ratios become unstable (i.e., large \(\epsilon\)).
Selecting \(N\) appropriately is therefore essential for obtaining an objective that is both stable and accurate enough to translate into strong empirical performance.

\subsection{Toy Example: Small Token MDP}
\noindent
\begin{minipage}{0.5\linewidth}
To illustrate the bias-variance trade-off of $N$-step surrogate, we consider a token MDP with a small vocabulary $\mathcal{V}=\{a,b,c\}$ and a fixed horizon $T=7$.
A trajectory $(y_1, \dots, y_T)$ is generated by a behavior policy $\mu$ and evaluated under a target policy $\pi$.
The reward is defined as a binary function of the completed sequence, indicating whether a predefined target pattern appears as an ordered subsequence. Further details on the token MDP can be found in Appendix~\ref{app:implementation details}.
The results clearly exhibit the predicted trade-off. As $N$ increases, the population surrogate objective $\Jcal(\mu) + \Lcal^{(N)}_\mu(\pi)$ approaches the true objective $\Jcal(\pi)$, confirming that the truncation bias decreases with $N$.
\end{minipage}
\hfill
\begin{minipage}{0.48\linewidth}
  \centering
  \includegraphics[width=\linewidth]{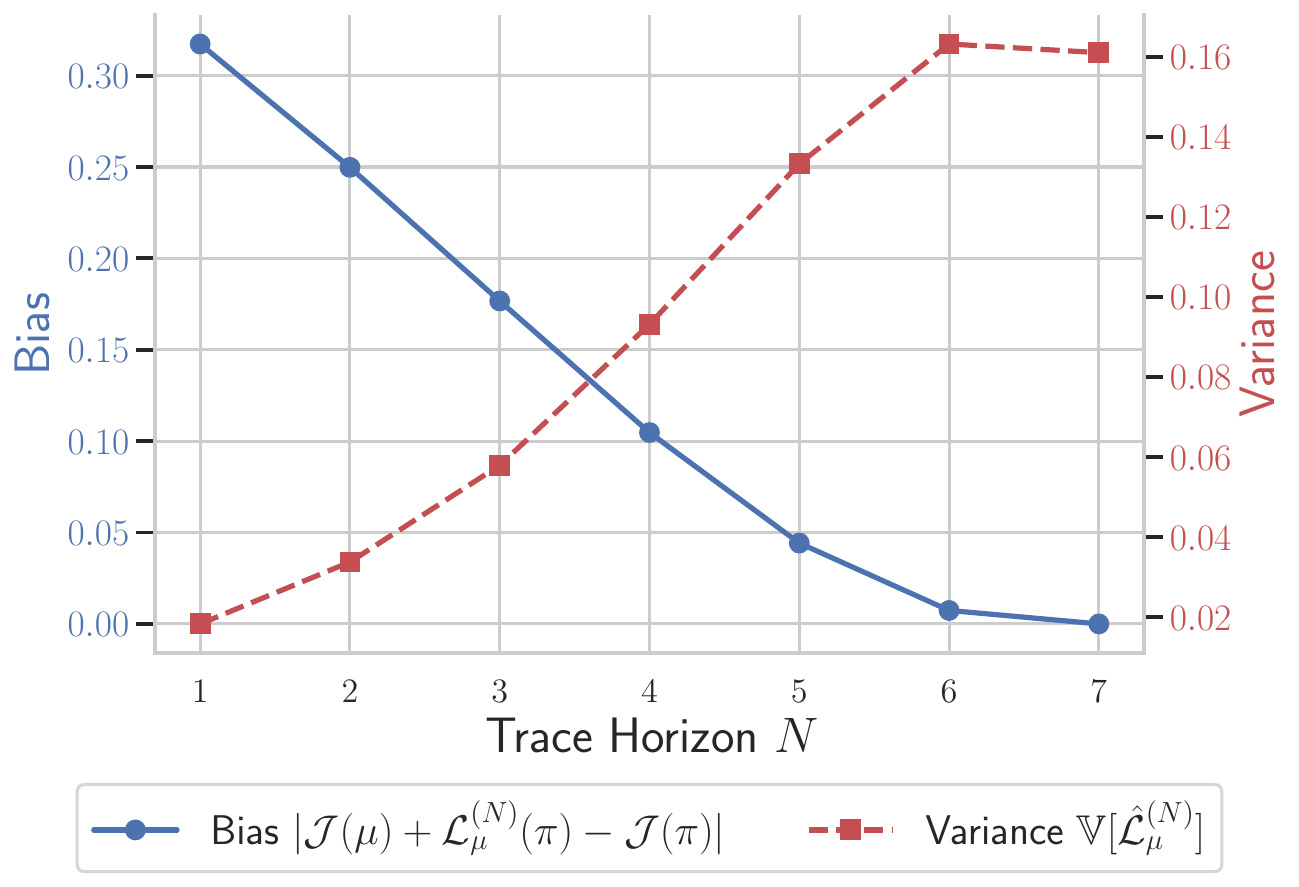}
  \captionof{figure}{\small{Bias--variance trade-off of the \(N\)-step surrogate objective on the token MDP example.}}
  \label{fig:bias_var_trade_off}
\end{minipage}

However, at the same time, the variance $\VV[\hat{\Lcal}^{(N)}_{\mu}(\pi)]$ increases with $N$, reflecting the growing instability introduced by longer products of importance ratios.
Overall, this example provides a clean and controlled validation of the theory: small $N$ yields a stable but biased objective, while large $N$ reduces bias at the cost of increased variance.

\subsection{Practical Implementation}

We now introduce \textbf{\underline{$\mathbf{N}$}-Step \underline{F}orward-Trace \underline{P}olicy \underline{O}ptimization (NFPO)}, a practical RLVR algorithm based on the $N$-step forward trace.
As indicated by Theorem~\ref{thm:n_step_policy_improvement}, a trust region is required for the $N$-step surrogate objective in \eqref{eq:N step surrogate loss} to guarantee policy improvement.
To enable a practical implementation of the trust region, we build upon the masked policy gradient framework and incorporate an $N$-step forward trace.
In practice, we apply stop-gradient to the forward trace for numerical stability, which leads to an objective based on $\rho_t$ rather than $(\rho_t - 1)$:
\begin{align*}
    \widehat \Lcal_{\mu,\mathrm{\AlgName{}}}^{(N)}(\pi)
    :=
    \frac1G
    \sum_{i=1}^{G}
    \sum_{t=1}^{T}
    M_t^{(i)}
    \cdot \widehat A^{(i)}
    \cdot \rho_t^{(i)}
    \cdot \bar\Gamma_{t+1}^{(N,i)},
\end{align*}
where $\bar\Gamma_{t+1}^{(N,i)}$ denotes the clipped $N$-step forward trace defined below.
To satisfy the bounded-ratio condition in Theorem~\ref{thm:n_step_policy_improvement} while ensuring numerical stability, we clip the token-level likelihood ratio as
$\rho_{t,\beta}^{(i)}
    :=
    \operatorname{clip}
    \big(
        \rho_t^{(i)},
        1/\beta,
        \beta
    \big).$
Based on this, we define the clipped $N$-step forward trace as
\[
    \bar\Gamma_{t+1}^{(N,i)}
    :=
    \operatorname{clip}
    \left(
        \prod_{j=t+1}^{m_t(N)}
        \rho_{j,\beta}^{(i)},
        1 - \epsilon_{\mathrm{low}},
        1 + \epsilon_{\mathrm{high}}
    \right).
\]
Finally, since the lower bound in Theorem~\ref{thm:n_step_policy_improvement} is expressed in terms of a total variation constraint, we adopt the token mask based on total variation from DPPO~\citep{qi2026rethinking}, where 
$M_t^{(i)}
    :=
    \mathbf{1}
    \big\{
       D_{\mathrm{TV}}
       \big(
           \mu(\cdot| s_t^{(i)}),
           \pi(\cdot| s_t^{(i)})
       \big)
       \le \delta
       \,\text{or}\,
       \widehat A^{(i)}
       \bigl(
           \rho_t^{(i)}-1
       \bigr)
       \le 0
    \big\}.$


\section{Main Experiments on LLM Reasoning Benchmarks}
\label{sec:main_results}
In this section, we empirically evaluate \AlgName{} on LLM reasoning tasks.

\textbf{Models and baselines. }\,
We employ Qwen3-1.7B-Base~\citep{yang2025qwen3} and Qwen3-8B-Base as base models. We consider two RLVR algorithms to compare with the performance of our algorithm: GRPO, the standard RLVR algorithm that utilizes group-relative advantage estimation and PPO-style clip, and DPPO, which applies divergence-aware token mask to implement trust-region.

\textbf{Training and evaluation. }\,
We use the level 3-5 problems of MATH~\citep{hendrycks2021measuring} dataset with 8k maximum response length for Qwen3-1.7B-Base, and DAPO-Math-17k dataset~\citep{yu2025dapo} with 16k response length for Qwen3-8B-Base.
We sample 128 prompts per step, and generate 8 responses per prompt. After generating the 1024 responses, the policy is updated using mini-batches of size 256.
More details are in Appendix~\ref{app:implementation details}.
The performance is measured by seven mathematical reasoning benchmarks: AIME24/25/26~\citep{aime_2026}, AMC23~\citep{amc_2023}, MATH 500~\citep{hendrycks2021measuring,lightman2023let}, Minerva~\citep{lewkowycz2022solving}, and OlympiadBench (Oly.)~\citep{he2024olympiadbench}.
We report the pass@1 metric averaged over 32 responses per prompt, with sampling temperature $0.7$.

\textbf{Performance evaluation. }\,
Table~\ref{tab:main} presents the evaluation results on comprehensive math reasoning benchmarks.
\AlgName{} consistently outperforms or matches baseline algorithms in almost every benchmark.
In particular, Qwen3-1.7B-Base model trained by \AlgName{} achieves $30.7$\% accuracy on average, which is $16.3$\% improvement compared to GRPO, and $10.0$\% compared to DPPO.
The same trend persists in the scaling experiment with Qwen3-8B-Base model, where \AlgName{} continues to outperform the baselines across most evaluation benchmarks.
These results suggest that the advantage of forward trace correction remains consistent across both model scales and diverse reasoning benchmarks.
Detailed training metrics are presented in Appendix~\ref{app:detailed experimental results}.

\begin{table}[t]
    \centering
    \caption{
    \small{Performance comparison on mathematical reasoning benchmarks. We report pass@1 averaged over 32 responses per prompt.}}
    
    \setlength{\tabcolsep}{4pt}
    \small 
    
    \begin{tabular}{@{}l | c c c c c c c | c@{}}
        \toprule
        \textbf{} & AIME24 & AIME25 & AIME26 & AMC23 & MATH 500 & Minerva & Oly. & Avg. \\
        \midrule
        
        Qwen3-1.7B-Base & $0.034$ & $0.015$ & $0.021$ & $0.261$ & $0.485$ & $0.134$ & $0.184$ & $0.162$ \\
        \hspace{1em} GRPO & $0.094$ & $0.050$ & $0.047$ & $0.425$ & $0.690$ & $0.225$ & $0.315$ & $0.264$ \\
        \hspace{1em} DPPO & $0.106$ & $0.066$ & $0.056$ & $0.424$ & $0.715$ & $0.246$ & $0.341$ & $0.279$ \\
        \hspace{1em} \AlgName{} (ours) & $\mathbf{0.133}$ & $\mathbf{0.074}$ & $\mathbf{0.081}$ & $\mathbf{0.496}$ & $\mathbf{0.745}$ & $\mathbf{0.262}$ & $\mathbf{0.360}$ & $\mathbf{0.307}$ \\

        \addlinespace[0.5ex] 
        \midrule

        Qwen3-8B-Base & $0.094$ & $0.074$ & $0.056$ & $0.423$ & $0.623$ & $0.219$ & $0.292$ & $0.254$ \\
        \hspace{1em} GRPO & $0.274$ & $0.229$ & $0.230$ & $0.741$ & $0.872$ & $0.351$ & $0.527$ & $0.461$  \\
        \hspace{1em} DPPO & $0.347$ & $0.256$ & $0.265$ & $0.791$ & $0.886$ & $\mathbf{0.399}$ & $0.565$ & $0.501$ \\
        \hspace{1em} \AlgName{} (ours) & $\mathbf{0.364}$ & $\mathbf{0.284}$ & $\mathbf{0.286}$ & $\mathbf{0.820}$ & $\mathbf{0.892}$ & $0.390$ & $\mathbf{0.571}$ & $\mathbf{0.515}$ \\
        
        \bottomrule
    \end{tabular}
    \label{tab:main}
\end{table}

\section{Learning-Dynamics Analysis of Forward Trace Correction}
\label{sec:analysis}
We further explore the effect of the forward trace on RL training.
This empirical analysis augments the theoretical findings in Section~\ref{sec:theoretical analysis} by revealing how forward trace changes the training dynamics.

\textbf{Observation 1: \AlgName{} corrects tokens more actively, especially reasoning tokens. }\,
Our token analysis shows that the forward trace enables more active bias correction, and the effect is more pronounced for \textit{reasoning tokens} such as ``suppose'', ``since'', and ``wait''.

\begin{wraptable}{r}{0.48\textwidth}
    \centering
    \vspace{-0.3cm}
    \caption{\small{Correction strength of likelihood ratio and trace-corrected likelihood ratio, measured for each token category.}}
    \begin{tabular}{l c c}
        \toprule
        \multirow{2}{*}{Category} & \multicolumn{2}{c}{Correction strength} \\
        \cmidrule(lr){2-3}
         & $\rho_t$ & $\rho_t \cdot \bar{\Gamma}^{(N,i)}_{t+1}$ \\
        \midrule
        All tokens         & $0.0278$ & \textbf{$0.0722$} \\
        Reasoning tokens   & $0.0531$ & \textbf{$0.1078$} \\
        \bottomrule
    \end{tabular}
    \vspace{-0.4cm}
    \label{table:correction_strength}
\end{wraptable}
The likelihood ratio $\rho_t$ corresponds to the token-level weight of standard PPO surrogate objectives (e.g. GRPO), whereas \AlgName{} augments each weight with the forward trace, yielding \emph{trace-corrected likelihood ratio} $\rho_t \cdot \bar{\Gamma}^{(N,i)}_{t+1}$.
During the RL training, tokens whose token-level weight deviates more from $1$ receive more active correction signals.
To examine how forward trace changes the dynamics of token-level correction, we measure correction strength for the two objectives: $\mathbb{E}_t[ |\rho_t - 1 |]$ represents the correction strength of the PPO surrogate objective, whereas $\mathbb{E}_t[ |\rho_t \cdot \bar{\Gamma}^{(N,i)}_{t+1} - 1 |]$ quantifies the correction strength of \AlgName{}.
We compute correction strength for approximately $3.4\mathrm{M}$ tokens collected during training, and categorize the tokens semantically to compare the correction strength across categories (see Appendix~\ref{app:token_full} for details of the procedure).

As shown in Table~\ref{table:correction_strength}, the correction strength of the trace-corrected likelihood ratio is substantially higher than that of the uncorrected likelihood ratio.
This implies that \AlgName{} more actively applies correction signals to reduce the bias from policy deviation.
In contrast, $\rho_t$ alone provides weak correction signals.
Moreover, the correction strength for reasoning tokens is much higher than the average for both objectives, showing that reasoning tokens are critical for bias correction.
This finding aligns with prior work showing that reasoning performance in RLVR is driven by these reasoning tokens that act as forks in the reasoning trajectory~\citep{wang2025beyond}.
In summary, the forward trace $\bar{\Gamma}^{(N,i)}_{t+1}$ enables more active bias correction, and such correction signals are even intensified for reasoning tokens.
This explains why \AlgName{} is more effective in learning complex reasoning tasks, under policy deviation.
\begin{figure}[t]
    \centering
    \includegraphics[width=0.9\linewidth]{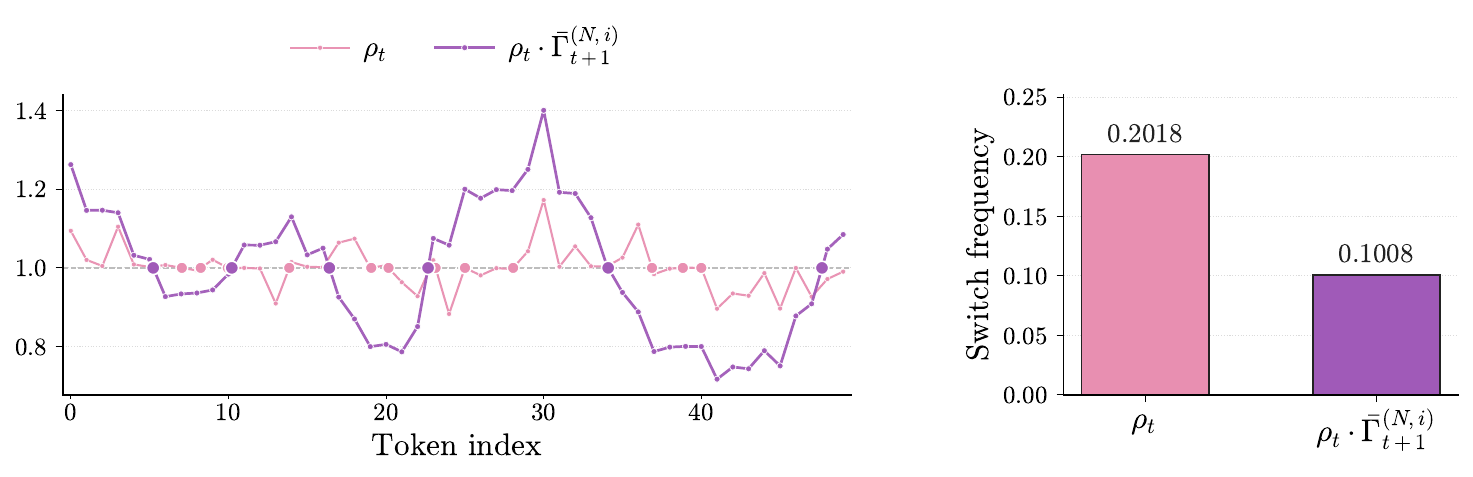}
    \caption{\small{
    (Left) Comparison of token-level ($\rho_t$) and trace-corrected ($\rho_t \cdot \bar{\Gamma}^{(N,i)}_{t+1}$) likelihood ratios along a 50-token rollout segment. (Right) Batch-averaged switch frequency, halved by the forward trace.}
    }
    \label{fig:observation2}
\end{figure}

\textbf{Observation 2: Forward trace acts as low-pass filter. }\,
The key effect of the $N$-step forward trace is to smooth noisy token-wise likelihood-ratio corrections into coherent segment-level update signals.
As shown in the left panel of Figure~\ref{fig:observation2}, the token-level likelihood ratio $\rho_t$ can oscillate rapidly around the neutral value of $1$ across neighboring token positions.
Such high-frequency switching can harm optimization because it breaks local coherence and distorts the policy gradient updates of neighboring tokens, ultimately destabilizing training~\citep{he2026kpo}.
In contrast, \AlgName{} uses the trace-corrected ratio $\rho_t \cdot \bar{\Gamma}^{(N,i)}_{t+1}$, where the forward trace aggregates the likelihood ratios of the subsequent $N-1$ tokens.
This aggregation couples the correction signals of nearby tokens, since adjacent positions share most of their $N$-step forward trace.
As a result, isolated token-level fluctuations are suppressed, while persistent signals over local segments are preserved.

To quantify this smoothing effect, we measure the switch frequency, defined as the fraction of adjacent token pairs whose ratio signals lie on opposite sides of the neutral value \(1\).
A high switch frequency indicates rapid alternation between up-weighting and down-weighting neighboring tokens, while a low switch frequency reflects more coherent segment-level correction.
As shown in the right panel of Figure~\ref{fig:observation2}, the forward trace roughly halves the batch-averaged switch frequency, reducing it from \(20.18\%\) for \(\rho_t\) to \(10.08\%\) for \(\rho_t \cdot \bar{\Gamma}^{(N,i)}_{t+1}\).
This confirms that the forward trace suppresses high-frequency token-level oscillations and acts as a low-pass filter over likelihood-ratio corrections.

\begin{figure}
    \centering
    \includegraphics[width=1.0\linewidth]{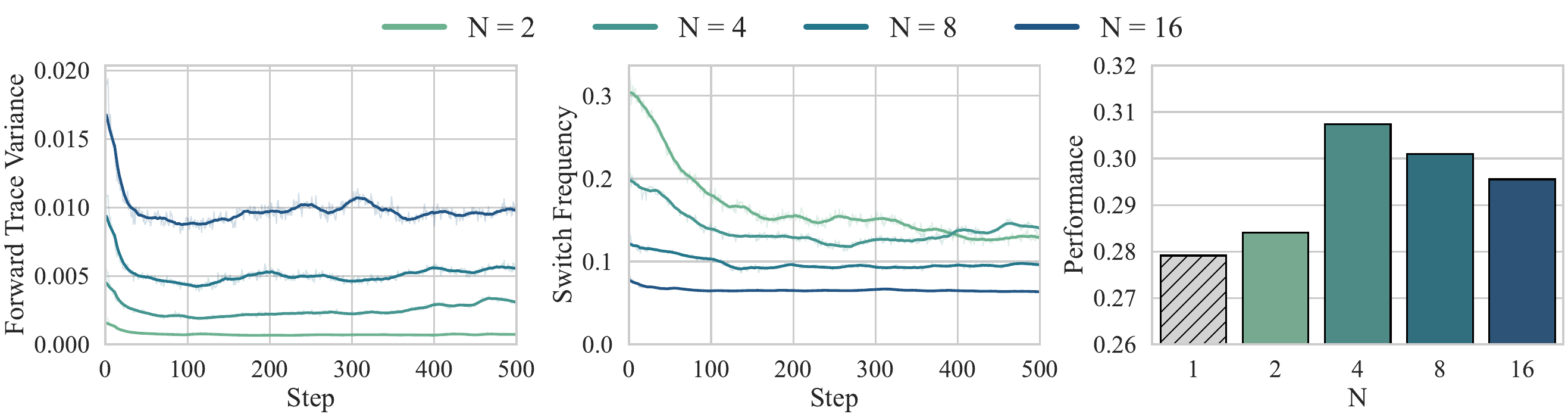}
    \caption{\small{Ablation results across $N$, where the $N=1$ corresponds to DPPO. (Left) Variance of forward trace. (Middle) Switch frequency of forward trace. (Right) Performance on the benchmarks.}}
    \label{fig:ablation n}
\end{figure}
\textbf{Observation 3: An intermediate $N$ achieves the best performance. }\,
We analyze the effect of $N$, which determines the accumulation length of the forward trace.
The results in Figure~\ref{fig:ablation n} reveal a clear pattern.
In terms of performance, $N=4$ emerges as a \textit{sweet spot}, while all configurations with $N>1$ consistently outperform the $N=1$ baseline.

The existence of such a sweet spot becomes evident when examining the training dynamics.
As $N$ grows, the variance of the forward trace increases.
While a moderate level of variance reflects active correction and is therefore beneficial, excessive variance can introduce noise into the gradient estimates and degrade training stability.
At the same time, increasing $N$ enhances semantic low-pass filtering, which is indicated by the reduced switch frequency of $\rho_t \cdot \bar{\Gamma}^{(N)}_{t+1}$.
These two effects evolve in opposing directions as $N$ varies, giving rise to an intermediate regime where the variance in forward trace and the strength of the low-pass filtering are optimally balanced.
This interplay naturally leads to the emergence of a performance sweet spot.

\begin{figure}
    \centering
    \includegraphics[width=1.0\linewidth]{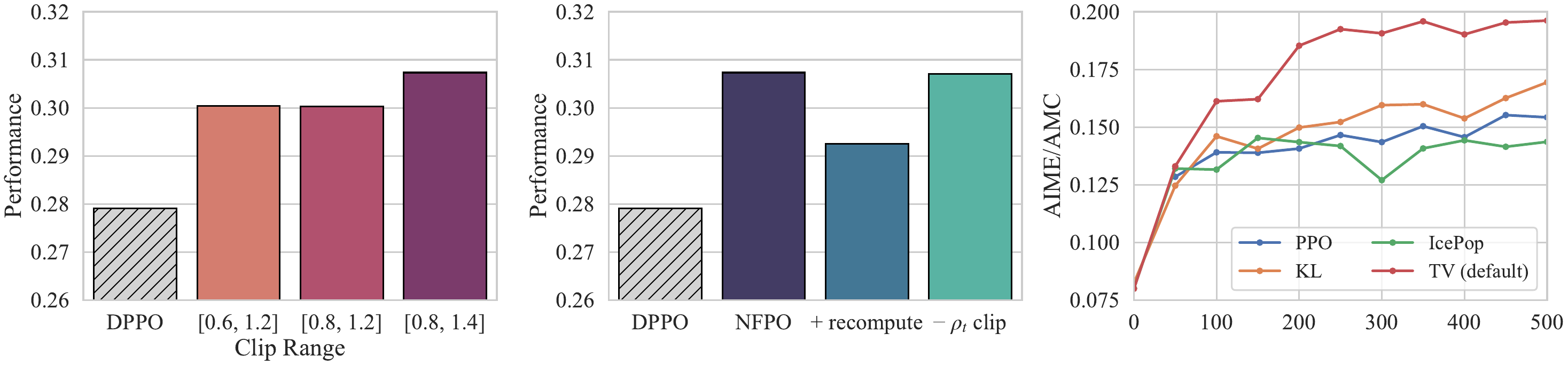}
    \caption{\small{Ablation experiments on NFPO: (Left) Different clip ranges, (Middle) Alternative design choices, and (Right) Different token masks.}}
    \label{fig:ablation misc}
\end{figure}

\section{Ablation Study}
\label{sec:ablation}
We present a comprehensive ablation study on the hyperparameters and design choices of \AlgName{}.

\textbf{Ablation 1. Forward trace clip range.}\,
We investigate the impact of the clip range of forward trace.
Using the clip range of GRPO as a reference, we start with the $[0.8,1.2]$ range and compare with the clip-higher (increasing $\epsilon_\mathrm{high}$ to $1.4$) and clip-lower (decreasing $\epsilon_\mathrm{low}$ to $0.6$).
The left plot in Figure~\ref{fig:ablation misc} shows that all considered clipping ranges outperform the DPPO baseline, demonstrating that the performance improvement of \AlgName{} is robust to the clip range.

\textbf{Ablation 2. Log probability recompute.}\,
Training–inference mismatch introduces a design choice regarding whether to recompute the rollout policy’s log probabilities within the training backend.
Since forward trace computation also relies on the rollout log probabilities, we evaluate the effect of recomputing them during training.
The middle plot in Figure~\ref{fig:ablation misc} shows that recompute leads to a slight performance degradation.
Considering the extra computational overhead required, we implement \AlgName\, without log probability recompute by default.

\textbf{Ablation 3. Likelihood ratio clip.}\,
In our implementation of \AlgName, we clip the likelihood ratio $\rho_{t}$ for numerical stability.
This clipping serves as a minimal safeguard to prevent a small number of outliers from disrupting the training dynamics; accordingly, the clipping threshold is fixed at $\beta=3.0$ across all experiments without tuning.
The middle plot in Figure~\ref{fig:ablation misc} shows that removing this clipping still yields a similar performance.
This behavior is likely due to the clipping applied to the forward trace itself, which prevents variance from exploding, and suggests that the effectiveness of the forward trace is not particularly sensitive to clipping applied to $\rho_t$.

\textbf{Ablation 4. Token mask.}\,
\AlgName\, is built upon the masked policy gradient framework, where the trust region is enforced via token-level masking.
In our main experiments, we adopt the token mask from DPPO, which implements a trust region defined by total variation.
However, \AlgName\, is not restricted to this specific choice.
The right plot in Figure~\ref{fig:ablation misc} presents the performance of \AlgName{} with various token masks.
The total variation-based mask~\citep{qi2026rethinking} outperforms the KL divergence based mask~\citep{qi2026rethinking}, the PPO clipping used in GRPO, and the token mask employed in IcePop~\citep{team2025every}.
Accordingly, we adopt the TV–based token mask as the default choice.

\section{Conclusion}
\label{sec:conclusion}

This work introduces the concept of the $N$-step forward trace, which 
bridges the widely used local surrogate objectives in RLVR and the exact policy gradient objective.
We show that forward trace correction provides a principled mechanism for interpolating between local surrogate losses and exact policy gradients, and through both theoretical analysis and large-scale LLM experiments, we demonstrate that an appropriate choice of the trace horizon $N$ can consistently improve performance over conventional PPO/GRPO-style objectives.
Furthermore, our mechanistic analyses provide intuitive insight into how forward trace correction improves learning dynamics in practice, complementing and reinforcing our theoretical findings.

\bibliographystyle{plainnat}
\bibliography{ref}

\newpage
\appendix
\onecolumn
\counterwithin{table}{section}
\counterwithin{lemma}{section}
\counterwithin{corollary}{section}
\counterwithin{theorem}{section}
\counterwithin{algorithm}{section}
\counterwithin{assumption}{section}
\counterwithin{figure}{section}
\counterwithin{equation}{section}
\counterwithin{condition}{section}
\counterwithin{remark}{section}
\counterwithin{definition}{section}
\counterwithin{proposition}{section}

\section{Related Works}
\label{sec:related}

\paragraph{Reinforcement learning with verifiable rewards (RLVR).}
Reinforcement learning has become a central paradigm for post-training large language models~\citep{ouyang2022training,rafailov2023direct,openo1,guo2025deepseek}.
In particular, for reasoning tasks such as mathematics and coding, reinforcement learning with verifiable rewards (RLVR)—where rewards are assigned based on objective, programmatically verifiable outcomes—has demonstrated notable success.
Building on the theoretical foundations of Trust Region Policy Optimization (TRPO)~\citep{schulman2015trust} and Proximal Policy Optimization (PPO)~\citep{schulman2017proximal}, a range of methods tailored to large-scale language model training have been developed.

Group Relative Policy Optimization (GRPO)~\citep{shao2024deepseekmath,guo2025deepseek} eliminates the need for a learned critic and instead employs group-relative advantage estimation.
This approach has been shown to achieve strong empirical performance while reducing cost, and has since emerged as a standard framework in RLVR training of reasoning language models.
Subsequent works have further improved this framework along multiple directions, including more effective advantage computation~\citep{yu2025dapo,lin2026cppo,wu2025quantile,cheng2026reasoning}, sampling strategies~\citep{yu2025dapo,zhang2025srpo,xiong2025minimalist,zhu2026the}, entropy regularization~\citep{cui2025entropy,petrenko2026entropy}, enhanced training stability~\citep{zheng2025group,chen2025minimax,team2025every,zheng2025stabilizing,qi2025defeating,zhao2025geometric,gao2025soft,qi2026rethinking}, and improved learning objective~\citep{liu2025understanding,chu2026gpg,wang2025aspo,chen2025pass,tajwar2026maximum}.

\paragraph{Surrogate objectives for policy gradient.}
When the behavior policy used for data collection is close to the training policy, the policy gradient objective can be locally approximated, yielding a surrogate objective that enables off-policy optimization without requiring data to be regenerated under the evolving policy.
Conservative Policy Iteration (CPI)~\citep{kakade2002approximately} derives policy improvement guarantees using such a surrogate objective under a policy-mixture update rule, while TRPO~\citep{schulman2015trust} enforces a trust-region constraint to ensure the validity of the local approximation during policy updates.
However, the update rule of CPI is difficult to implement for complex policy classes, and although TRPO is more practical, it still requires second-order optimization.

PPO~\citep{schulman2017proximal} implements the trust-region concept of TRPO via simple clipping operations, making it applicable to complex policy classes.
As a result, most algorithms used in RLVR training of large language models, including GRPO~\citep{shao2024deepseekmath,guo2025deepseek}, are developed as variants of PPO.
To address domain-specific challenges in LLM training, several clipping strategies and token-level masking techniques have been proposed~\citep{yu2025dapo,zheng2025group,gao2025soft,team2025every,zhao2025geometric,qi2026rethinking}.
Nevertheless, the majority of these methods are fundamentally built upon the surrogate objective derived from CPI and TRPO, i.e., a locally approximated policy gradient objective.

In this work, we revisit this local approximation and show that a surrogate objective corrected by an $N$-step forward trace bridges the fully local surrogate and the exact policy gradient.
Our contribution therefore does not lie in refining the implementation of PPO-style objectives, but in modifying the objective itself to achieve improved performance.

\paragraph{Connection with FIPO~\citep{ma2026fipo}.}
Very recently,~\citet{ma2026fipo} proposed Future-KL Influenced Policy Optimization (FIPO), a future-dependent token reweighting method that accumulates \textit{discounted} log-probability shifts over the remaining tokens to modulate token-level policy updates.
This also induces a future-dependent multiplicative weight on token-level updates, making FIPO superficially related to our forward trace formulation.
The key difference is conceptual and theoretical.
FIPO is motivated as a \textit{heuristic} dense-credit-assignment rule based on future-KL signals, 
while our method is derived from the \textit{exact policy-improvement identity}, where the forward trace appears as the principled likelihood-ratio correction term.
Consequently, our objective is not designed as an ad hoc advantage reweighting scheme; rather, it corrects the structural bias of local RLVR objectives and exposes a formal bias--variance trade-off governed by the trace horizon \(N\).

Moreover, when the discount factor in FIPO is set to one, its future-dependent reweighting can be interpreted as the full-horizon instance of our framework, corresponding to the forward trace with \(N=T\).
This highlights the generality of our formulation: it explains existing future-dependent reweighting structures while also providing a principled truncated family with an explicit bias--variance characterization.

For completeness, we also compare our method empirically with FIPO.
The two methods differ not only in how future-dependent signals are incorporated, but also in implementation details such as the trust-region mechanism.
As shown in Figure~\ref{fig:NFPO vs FIPO}, our method consistently outperforms FIPO by a substantial margin.
Rather than attributing this gap to a single isolated factor, we view it as the result of a theoretically grounded design that better realizes the benefits of forward trace correction.
In particular, the masked policy-gradient framework with a total-variation-based token mask provides a reliable way to enforce the trust-region condition in Theorem~\ref{thm:n_step_policy_improvement}.
\begin{figure}[h]
    \centering
    \includegraphics[width=1.0\linewidth]{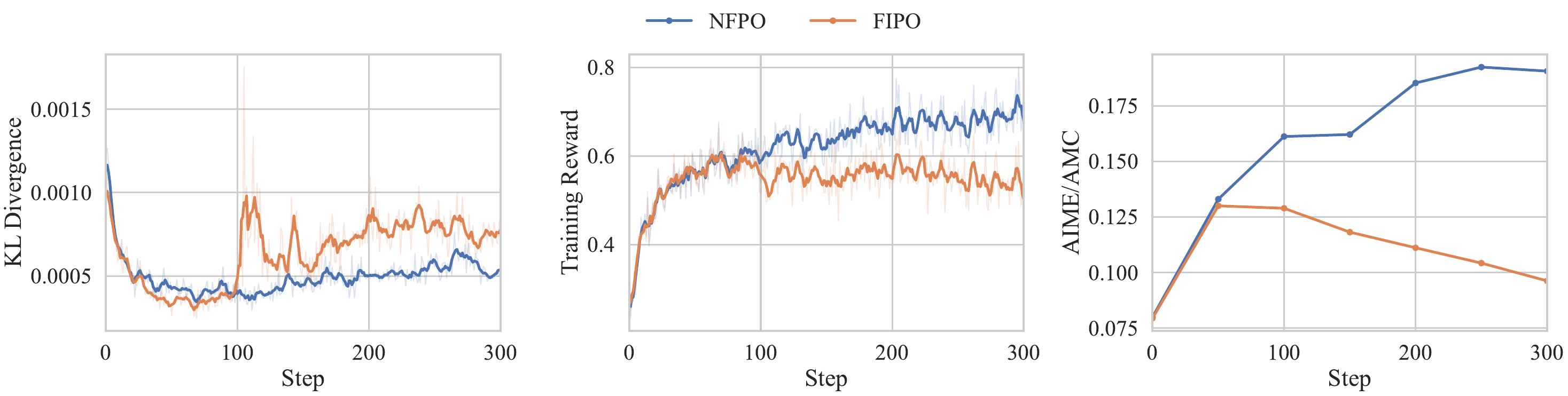}
    \caption{(Left) KL divergence between training and rollout policy, (middle) training reward, and (right) average evaluation performance on AIME24/25/26 and AMC23 for NFPO and FIPO. The base model is Qwen3-1.7B-Base model, and the training setting follows that described in Section~\ref{sec:main_results}.}
    \label{fig:NFPO vs FIPO}
\end{figure}

\section{Detailed Proofs}
\label{app:detailed proofs}

\subsection{Proof of Performance Difference Lemma} \label{appsubsec:proof_lemma:pdl}
\begin{proof} [Proof of Lemma~\ref{lemma:pdl}]
    By definition, the performance difference can be written as
    \begin{align}
        \Jcal(\pi)-\Jcal(\mu)
        =
        \sum_y \bigl(\pi(y)-\mu(y)\bigr)R(y).
        \label{eq:lemma:pdl_gap}
    \end{align}
    Using the autoregressive factorization,
    we have
    $\pi(y)=\prod_{t=1}^T \pi(y_t| s_t)$
    and
    $\mu(y)=\prod_{t=1}^T \mu(y_t| s_t).$
    Then, by the telescoping identity, we get
    \[
    \pi(y)-\mu(y)
    =
    \sum_{t=1}^T
    \left(\prod_{k=1}^{t-1}\mu(y_k| s_k)\right)
    \bigl(\pi(y_t| s_t)-\mu(y_t| s_t)\bigr)
    \left(\prod_{j=t+1}^T \pi(y_j| s_j)\right).
    \]
    Substituting this identity into~\eqref{eq:lemma:pdl_gap} yields
    \[
    \Jcal(\pi)-\Jcal(\mu)
    =
    \sum_y R(y)\sum_{t=1}^T
    \left(\prod_{k=1}^{t-1}\mu(y_k| s_k)\right)
    \bigl(\pi(y_t| s_t)-\mu(y_t| s_t)\bigr)
    \left(\prod_{j=t+1}^T \pi(y_j| s_j)\right).
    \]
    Now multiply and divide by \(\mu(y_t| s_t)\prod_{j=t+1}^T \mu(y_j| s_j)\), and note that
    \[
    \mu(y)
    =
    \left(\prod_{k=1}^{t-1}\mu(y_k| s_k)\right)\mu(y_t| s_t)\prod_{j=t+1}^T \mu(y_j| s_j).
    \]
    Hence,
    \[
    \Jcal(\pi)-\Jcal(\mu)
    =
    \sum_y \mu(y)R(y)\sum_{t=1}^T
    \left(
    \frac{\pi(y_t| s_t)}{\mu(y_t| s_t)}-1
    \right)
    \prod_{j=t+1}^T
    \frac{\pi(y_j| s_j)}{\mu(y_j| s_j)}.
    \]
    Using the definitions \(\rho_t= \frac{\pi(y_t| s_t)}{\mu(y_t| s_t)}\) and
    \(\Gamma_{t+1}=\prod_{j=t+1}^T \rho_j\), we obtain
    \[
    \Jcal(\pi)-\Jcal(\mu)
    =
    \EE_{y\sim\mu}
    \left[
    R(y)\sum_{t=1}^T (\rho_t-1)\Gamma_{t+1}
    \right],
    \]
    which proves the claim.
\end{proof}

\subsection{Proof of Policy Improvement}
\label{app_subsec:policy_improve}
\begin{proof} [Proof of Proposition~\ref{prop:policy_improvement_bound}]
    The first inequality follows directly from Theorem~3.2 of~\citep{qi2026rethinking}; we include the proof here for completeness.
    Fix a policy \(\pi\), and write
        $\delta
        :=
        D_{\mathrm{TV}}^{\max}(\mu,\pi)
        =
        \sup_s
        D_{\mathrm{TV}}\!\left(
            \mu(\cdot | s),\pi(\cdot | s)
        \right)$.
    By the performance difference lemma (Lemma~\ref{lemma:pdl}), we have
    \[
        \Jcal(\pi)-\Jcal(\mu)
        =
        \EE_{y\sim\mu}
        \left[
            R(y)\sum_{t=1}^{T}
            \bigl(\rho_t-1\bigr)\Gamma_{t+1}
        \right],
    \]
    where \(\Gamma_{t+1}:=\prod_{j=t+1}^{T}\rho_j\). 
    Therefore, we get
    \begin{align*}
        \Jcal(\pi)-\Jcal(\mu)
        &=
        \EE_{y\sim\mu}
        \left[
            R(y)\sum_{t=1}^{T}
            \bigl(\rho_t-1\bigr)
        \right]
        +
        \EE_{y\sim\mu}
        \left[
            R(y)\sum_{t=1}^{T}
            \bigl(\rho_t-1\bigr)
            \bigl(\Gamma_{t+1}-1\bigr)
        \right]
        \\
        &=
        \Lcal_{\mu}(\pi)
        +
        \EE_{y\sim\mu}
        \left[
            R(y)\sum_{t=1}^{T}
            \bigl(\rho_t-1\bigr)
            \bigl(\Gamma_{t+1}-1\bigr)
        \right].
        \numberthis
        \label{eq:pi_decomposition}
    \end{align*}
    Using \(|R(y)|\le \xi\), it remains to bound the second term from below:
    \begin{equation}
        \EE_{y\sim\mu}
        \left[
            R(y)\sum_{t=1}^{T}
            \bigl(\rho_t-1\bigr)
            \bigl(\Gamma_{t+1}-1\bigr)
        \right]
        \ge
        -\xi
        \sum_{t=1}^{T}
        \EE_{y\sim\mu}
        \left[
            |\rho_t-1|\,
            |\Gamma_{t+1}-1|
        \right].
        \label{eq:pi_decomposition_second}
    \end{equation}

    We first record two elementary bounds. 
    For any state \(s\),
    \begin{equation}
    \label{eq:pi_first}
        \EE_{a\sim\mu(\cdot | s)}
        \left[
            \left|
            \frac{\pi(a| s)}{\mu(a| s)}-1
            \right|
        \right]
        =
        \sum_a
        \left|
            \pi(a| s)-\mu(a| s)
        \right|
        =
        2D_{\mathrm{TV}}\!\left(
            \mu(\cdot | s),\pi(\cdot | s)
        \right)
        \le 2\delta.
    \end{equation}
    Next, conditional on \(s_{t+1}=(x,y_{\le t})\), the forward trace
\(\Gamma_{t+1}(\theta)\) is the likelihood ratio of the future-token sequence
\(y_{t+1:T}\) under \(\pi\) relative to \(\mu\).
    Hence,
    \[
        \EE_{\mu}
        \left[
            |\Gamma_{t+1}-1|
            \mid s_{t+1}
        \right]
        =
        2D_{\mathrm{TV}}\!\left(
            P_{\pi}^{t+1:T}(\cdot | s_{t+1}),
            P_{\mu}^{t+1:T}(\cdot | s_{t+1})
        \right).
    \]
    By a standard telescoping/coupling bound for sequential distributions, we get
    \[
        D_{\mathrm{TV}}\!\left(
            P_{\pi}^{t+1:T}(\cdot | s_{t+1}),
            P_{\mu}^{t+1:T}(\cdot | s_{t+1})
        \right)
        \le
        \sum_{j=t+1}^{T}
        D_{\mathrm{TV}}^{\max}(\mu,\pi)
        =
        (T-t)\delta.
    \]
    Therefore, we obtain
    \begin{equation}
    \label{eq:pi_second}
        \EE_{\mu}
        \left[
            |\Gamma_{t+1}-1|
            \mid s_{t+1}
        \right]
        \le
        2(T-t)\delta.
    \end{equation}

    Combining the two bounds, for each \(t\le T-1\), we have
    \begin{align*}
        \EE_{\mu}
        \left[
            |\rho_t-1|\,
            |\Gamma_{t+1}-1|
        \right]
        &=
        \EE_{\mu}
        \left[
            |\rho_t-1|\,
            \EE_{\mu}
            \left[
                |\Gamma_{t+1}-1|
                \mid s_{t+1}
            \right]
        \right]
        \\
        &\le
        2(T-t)\delta\,
        \EE_{\mu}
        \left[
            |\rho_t-1|
        \right]
        \tag{Eqn.~\eqref{eq:pi_second}}
        \\
        &\le
        4(T-t)\delta^2.
        \tag{Eqn.~\eqref{eq:pi_first}}
    \end{align*}
    For \(t=T\), let \(\Gamma_{T+1}=1\), so the corresponding term is zero.
    Thus,
    \begin{align*}
        \sum_{t=1}^{T}
        \EE_{\mu}
        \left[
            |\rho_t-1|\,
            |\Gamma_{t+1}-1|
        \right]
        &\le
        4\delta^2
        \sum_{t=1}^{T-1}(T-t)
        \\
        &=
        4\delta^2\cdot \frac{T(T-1)}{2}
        \\
        &=
        2T(T-1)\delta^2.
    \end{align*}
    Combining this bound with~\eqref{eq:pi_decomposition} and~\eqref{eq:pi_decomposition_second} gives
    \begin{equation}
        \Jcal(\pi)-\Jcal(\mu)
        \ge
        \Lcal_{\mu}(\pi)
        -
        2\xi T(T-1)
        \Bigl(D_{\mathrm{TV}}^{\max}(\mu,\pi)\Bigr)^2,
        \label{eq:pi_lower_first}
    \end{equation}
    which proves the first lower bound.
    
    It remains to relate the population surrogate to its empirical counterpart. Define the
    single-sample surrogate estimator
    \[
        Z_\pi(y)
        :=
        R(y)\sum_{t=1}^{T}
        \bigl(\rho_t-1\bigr).
    \]
    Then \(\Lcal_{\mu}(\pi_\theta)=\EE_{\mu}[Z_\pi(y)]\) and
        $\widehat{\Lcal}_{\mu}(\pi_\theta)
        =
        \frac1G\sum_{i=1}^{G}Z_\pi(y^{(i)}).$
    By the bounded-reward assumption and the likelihood-ratio deviation bound, we get
    \begin{align*}
        |Z_\pi(y)|
        &=
        \left|
            R(y)
            \sum_{t=1}^{T}
            \bigl(\rho_t-1\bigr)
        \right|  
        \le
        |R(y)|
        \sum_{t=1}^{T}
        \left|\rho_t-1\right| 
        \le
        \xi T\epsilon .
    \end{align*}
     Since
    \(\{y^{(i)}\}_{i=1}^{G}\) are i.i.d. samples from \(\mu(\cdot| x)\), the random variables
    \(\{Z_\pi(y^{(i)})\}_{i=1}^{G}\) are also i.i.d. and bounded. By Hoeffding's inequality,
    with probability at least \(1-\alpha\), we have
    \[
    \EE_{\mu}[Z_\pi(y)]
    \ge
    \frac1G\sum_{i=1}^{G} Z_\pi(y^{(i)})
    -
    \xi T\epsilon
    \sqrt{
        \frac{2\log(1/\alpha)}{G}
    } .
    \]
    Equivalently,
    \[
        \Lcal_{\mu}(\pi)
        \ge
        \widehat{\Lcal}_{\mu}(\pi)
        -
        \xi T\epsilon
        \sqrt{
            \frac{2\log(1/\alpha)}{G}
        } .
    \]
    Combining this concentration bound with the deterministic lower bound in~\eqref{eq:pi_lower_first}, we obtain
    \[
    \begin{aligned}
        \Jcal(\pi)-\Jcal(\mu)
        &\ge
        \Lcal_{\mu}(\pi)
        -
        2\xi T(T-1)
        \Bigl(D_{\mathrm{TV}}^{\max}(\mu,\pi)\Bigr)^2 \\
        &\ge
        \widehat{\Lcal}_{\mu}(\pi)
        -
        2\xi T(T-1)
        \Bigl(D_{\mathrm{TV}}^{\max}(\mu,\pi)\Bigr)^2
        -
        \xi T\epsilon
        \sqrt{
            \frac{2\log(1/\alpha)}{G}
        } .
    \end{aligned}
    \]
    This completes the proof of Proposition~\ref{prop:policy_improvement_bound}.
\end{proof}

\subsection{Proof of Theorem~\ref{thm:n_step_policy_improvement}}
\label{app_sec:proof_main}
\begin{proof}[Proof of Theorem~\ref{thm:n_step_policy_improvement}]
    Fix a target policy \(\pi\) and a forward trace horizon
    \(N\in\{1,\ldots,T\}\).
    We prove the result in three steps.
    
    \paragraph{Step 1: Population truncation error.}
    Recall that, by Lemma~\ref{lemma:pdl}, we have
    \[
        \Jcal(\pi)-\Jcal(\mu)
        =
        \EE_{y\sim\mu}
        \left[
            R(y)
            \sum_{t=1}^{T}
            \bigl(\rho_t-1\bigr)
            \Gamma_{t+1}
        \right],
    \]
    where
        $\Gamma_{t+1}
        :=
        \prod_{j=t+1}^{T}
        \rho_j.$
    
    For each \(t\), define
        $m_t := \min\{t+N-1,T\}$.
    Then the full forward trace decomposes as
    \[
        \Gamma_{t+1}
        =
        \Gamma_{t+1}^{(N)}
        Q_{m_t+1},
    \]
    where
        $Q_{m_t+1}
        :=
        \prod_{j=m_t+1}^{T}
        \rho_j.$
    When \(m_t=T\), the residual product
    \(Q_{m_t+1}\) is an empty product and equals one.
    
    Therefore, we have
    \begin{align*}
        \Jcal(\pi)-\Jcal(\mu)
        -
        \Lcal_{\mu}^{(N)}(\pi)
        &= 
        \EE_{y\sim\mu}
        \left[
            R(y)
            \sum_{t=1}^{T}
            \bigl(\rho_t-1\bigr)
            \left(
                \Gamma_{t+1}
                -  \Gamma_{t+1}^{(N)}
            \right)
        \right]
        \\
        &=
        \EE_{y\sim\mu}
        \left[
            R(y)
            \sum_{t=1}^{T}
            \bigl(\rho_t-1\bigr)
            \Gamma_{t+1}^{(N)}
            \left(
                Q_{m_t+1}-1
            \right)
        \right]
        \\
        &\le
        \xi
        \sum_{t=1}^{T}
        \EE_{\mu}
        \left[
            \left|\rho_t-1\right|
            \Gamma_{t+1}^{(N)}
            \left|
                Q_{m_t+1}-1
            \right|
        \right]
        ,
        \numberthis 
        \label{eq:thm_gap}
    \end{align*}
    where the last inequality follows from the assumption \(|R(y)|\le \xi\).
    
    We now bound each summand. 
    First, for any state \(s\), we get
    \begin{align*}
        \EE_{a\sim\mu(\cdot| s)}
        \left[
            \left|
            \frac{\pi(a| s)}
                 {\mu(a| s)}
            -1
            \right|
        \right]
        &=
        \sum_a
        \mu(a| s)
        \left|
            \frac{\pi(a| s)}
                 {\mu(a| s)}
            -1
        \right|
        =
        2D_{\mathrm{TV}}
        \left(
            \mu(\cdot| s),
            \pi(\cdot| s)
        \right)
        \\
        &\le
        2D_{\mathrm{TV}}^{\max}(\mu,\pi).
    \end{align*}
    For brevity, we write
        $\delta := D_{\mathrm{TV}}^{\max}(\mu,\pi).$
    Thus, we have
    \begin{equation}
        \EE_{y_t\sim\mu(\cdot| s_t)}
        \left[
            |\rho_t-1|
            \mid s_t
        \right]
        \le 2\delta.
        \label{eq:thm_ratio_bound}
    \end{equation}

    Second, the residual future-token likelihood ratio
        $Q_{m_t+1}
        =
        \prod_{j=m_t+1}^{T}
        \rho_j$
    is the likelihood ratio of the remaining tokens \(y_{m_t+1:T}\) under \(\pi\) relative to
    \(\mu\), conditional on the prefix up to token \(m_t\).
    Let
    \(P_{\mu}^{T-m_t}(\cdot | s_{m_t+1})\) and
    \(P_{\pi}^{T-m_t}(\cdot | s_{m_t+1})\) denote the conditional distributions of the
    future token sequence \(y_{m_t+1:T}\) under \(\mu\) and \(\pi\), respectively. Then,
    conditional on \(s_{m_t+1}\), we have
    \begin{align*}
        \EE_{\mu}
        \left[
            \left|
                Q_{m_t+1}-1
            \right|
            \,\middle|\,
            s_{m_t+1}
        \right]
        &=
        \EE_{\mu}
        \left[
            \left|
            \frac{
                \prod_{j=m_t+1}^{T}
                \pi(y_j| s_j)
            }{
                \prod_{j=m_t+1}^{T}
                \mu(y_j| s_j)
            }
            -1
            \right|
            \,\middle|\,
            s_{m_t+1}
        \right] \\
        &=
        \sum_{y_{m_t+1:T}}
        \left(
            \prod_{j=m_t+1}^{T}
            \mu(y_j| s_j)
        \right)
        \left|
            \frac{
                \prod_{j=m_t+1}^{T}
                \pi(y_j| s_j)
            }{
                \prod_{j=m_t+1}^{T}
                \mu(y_j| s_j)
            }
            -1
        \right| \\
        &=
        \sum_{y_{m_t+1:T}}
        \left|
            \prod_{j=m_t+1}^{T}
            \pi(y_j| s_j)
            -
            \prod_{j=m_t+1}^{T}
            \mu(y_j| s_j)
        \right| \\
        &=
        2D_{\mathrm{TV}}
        \left(
            P_{\mu}^{T-m_t}(\cdot | s_{m_t+1}),
            P_{\pi}^{T-m_t}(\cdot | s_{m_t+1})
        \right),
    \end{align*}
    where last equality follows from
        $D_{\mathrm{TV}}(P,Q)
        :=
        \frac12\sum_y |P(y)-Q(y)|.$
        
    By the standard telescoping bound for total variation over
    sequential kernels, we have
    \[
        D_{\mathrm{TV}}
        \left(
            P_{\mu}^{T-m_t}(\cdot | s_{m_t+1}),
            P_{\pi}^{T-m_t}(\cdot | s_{m_t+1})
        \right)
        \le
        (T-m_t)\delta.
    \]
    Consequently, we obtain
    \begin{equation}
        \EE_{\mu}
        \left[
            \left|
                Q_{m_t+1}-1
            \right|
            \,\middle|\,
            s_{m_t+1}
        \right]
        \le
        2(T-m_t)\delta.
        \label{eq:thm_residual_bound}
    \end{equation}

    Plugging~\eqref{eq:thm_residual_bound} into~\eqref{eq:thm_gap}, we have
    \begin{align*}
        \EE_{\mu}
        \left[
            \left|\rho_t(\theta)-1\right|
            \Gamma_{t+1}^{(N)}
            \left|
                Q_{m_t+1}-1
            \right|
        \right]
        &=
        \EE_{\mu}
        \left[
            \left|\rho_t(\theta)-1\right|
            \Gamma_{t+1}^{(N)}
            \EE_{\mu}
            \left[
                \left|
                    Q_{m_t+1}-1
                \right|
                \,\middle|\,
                s_{m_t+1}
            \right]
        \right]
        \\
        &\le
        2(T-m_t)\delta
        \,
        \EE_{\mu}
        \left[
            \left|\rho_t-1\right|
            \Gamma_{t+1}^{(N)}
        \right].
        \numberthis
        \label{eq:thm_gap_bound2}
    \end{align*}
    Now condition on \(s_t\). 
    Since
    \(\Gamma_{t+1}^{(N)}\) is the likelihood-ratio product over the future
    window \(t+1,\ldots,m_t\), its conditional expectation under \(\mu\), given
    \(s_{t+1}\), is one, i.e.,
    \[
        \EE_{\mu}
        \left[
            \Gamma_{t+1}^{(N)}
            \,\middle|\,
            s_{t+1}
        \right]
        =
        1.
    \]
    Therefore, we have
    \begin{align*}
        \EE_{\mu}
        \left[
            \left|\rho_t-1\right|
            \Gamma_{t+1}^{(N)}
            \,\middle|\,
            s_t
        \right]
        &=
        \EE_{y_t\sim\mu(\cdot | s_t)}
        \left[
            \left|\rho_t-1\right|
            \EE_{\mu}
            \left[
                \Gamma_{t+1}^{(N)}
                \,\middle|\,
                s_{t+1}
            \right]
            \,\middle|\,
            s_t
        \right]
        \\
        &=
        \EE_{y_t\sim\mu(\cdot | s_t)}
        \left[
            \left|\rho_t-1\right|
            \,\middle|\,
            s_t
        \right]
        \\
        &\le
        2\delta.
        \tag{Eqn.~\eqref{eq:thm_ratio_bound}}
    \end{align*}
    Combining this bound with~\eqref{eq:thm_gap_bound2}, we get
    \[
        \EE_{\mu}
        \left[
            \left|\rho_t-1\right|
            \Gamma_{t+1}^{(N)}
            \left|
                Q_{m_t+1}-1
            \right|
        \right]
        \le
        4(T-m_t)\delta^2.
    \]
    
    Summing over \(t\), we obtain
    \begin{align*}
        \left|
        \Jcal(\pi)-\Jcal(\mu)
        -
        \Lcal_{\mu}^{(N)}(\pi)
        \right|
        &\le
        4\xi\delta^2
        \sum_{t=1}^{T}
        (T-m_t)
        =
        4\xi\delta^2
        \sum_{t=1}^{T}
        (T-t-N+1)_+
        \\
        &=  4\xi\delta^2 \sum_{t=1}^{T-N}
        (T-N+1-t)
        \\
        &=  2\xi (T-N)(T-N+1)\delta^2
        .
    \end{align*}
    This implies that
    \[
        \Jcal(\pi)-\Jcal(\mu)
        \ge
        \Lcal_{\mu}^{(N)}(\pi)
        -
        2\xi (T-N)(T-N+1)
        \Bigl(D_{\mathrm{TV}}^{\max}(\mu,\pi)\Bigr)^2.
    \]
    The second term on the right-hand side is the bias upper bound, whose magnitude is strictly decreasing in \(N\). 
    This proves the first inequality.

    \paragraph{Step 2: Empirical concentration.}
    
    Define the single-sample \(N\)-step surrogate estimator
    \[
        Z_N(y)
        :=
        R(y)
        \sum_{t=1}^{T}
        \bigl(\rho_t-1\bigr)
        \Gamma_{t+1}^{(N)}.
    \]
    Then the population \(N\)-step surrogate and its empirical estimate can be written as
    \[
        \Lcal_{\mu}^{(N)}(\pi)
        =
        \EE_{y\sim\mu}[Z_N(y)],
        \qquad
        \text{and}
        \qquad
        \widehat{\Lcal}_{\mu}^{(N)}(\pi)
        =
        \frac1G\sum_{i=1}^{G}Z_N(y^{(i)}).
    \]
    
    Let
        $k_t(N):=\min\{N-1,T-t\}.$
    The product \(\Gamma_{t+1}^{(N)}\) contains exactly \(k_t(N)\) future
    likelihood-ratio factors. 
    Since
        $\left|\rho_t-1\right|\le \epsilon,$
    we have
        $0\le \rho_t\le 1+\epsilon.$
    Therefore, we get
    \[
        0
        \le
        \Gamma_{t+1}^{(N)}
        \le
        (1+\epsilon)^{k_t(N)}
    \]
    Using \(|R(y)|\le \xi\), we obtain
    \begin{align*}
        |Z_N(y)|
        &=
        \left|
            R(y)
            \sum_{t=1}^{T}
            \bigl(\rho_t-1\bigr)
            \Gamma_{t+1}^{(N)}
        \right|
        \\
        &\le
        |R(y)|
        \sum_{t=1}^{T}
        \left|\rho_t-1\right|
        \Gamma_{t+1}^{(N)}
        \\
        &\le
        \xi\epsilon
        \sum_{t=1}^{T}
        (1+\epsilon)^{\min\{N-1,T-t\}}
        =: B_N.
    \end{align*}
    Thus,
        $Z_N(y)\in[-B_N,B_N].$
    
    Since \(y^{(1)},\ldots,y^{(G)}\) are i.i.d. samples from \(\mu\), the random
    variables \(Z_N(y^{(1)}),\ldots,Z_N(y^{(G)})\) are also i.i.d. and bounded.
    By Hoeffding's inequality, for any \(a>0\), we get
    \[
        \PP
        \left(
            \widehat{\Lcal}_{\mu}^{(N)}(\pi)
            -
            \Lcal_{\mu}^{(N)}(\pi)
            \ge a
        \right)
        \le
        \exp\left(
            -\frac{G a^2}{2B_N^2}
        \right).
    \]
    Taking
        $a
        =
        B_N
        \sqrt{
            \frac{2\log(1/\alpha)}{G}
        },$
    we obtain that, with probability at least \(1-\alpha\),
    \[
        \Lcal_{\mu}^{(N)}(\pi)
        \ge
        \widehat{\Lcal}_{\mu}^{(N)}(\pi)
        -
        B_N
        \sqrt{
            \frac{2\log(1/\alpha)}{G}
        }.
    \]
    
    Combining this with the population truncation bound from Step 1 gives
    \begin{align*}
        \Jcal(\pi)-\Jcal(\mu)
        &\ge
        \widehat{\Lcal}_{\mu}^{(N)}(\pi)
        -
        2\xi (T-N)(T-N+1)
        \Bigl(D_{\mathrm{TV}}^{\max}(\mu,\pi)\Bigr)^2
        -
        B_N
        \sqrt{
            \frac{2\log(1/\alpha)}{G}
        }.
    \end{align*}
    This proves the high-probability empirical policy-improvement bound.
    
    \paragraph{Step 3: Monotonicity of the concentration penalty.}
    
    It remains to show that \(B_N\) is strictly increasing in \(N\). 
    Define
        $S_N
        :=
        \sum_{t=1}^{T}
        (1+\epsilon)^{\min\{N-1,T-t\}},$
    so that
        $B_N=\xi\epsilon S_N.$
    Then, for \(N=1,\ldots,T-1\), we have
    \begin{align*}
        S_{N+1}-S_N
        &=
        \sum_{t=1}^{T}
        \left(
            (1+\epsilon)^{\min\{N,T-t\}}
            -
            (1+\epsilon)^{\min\{N-1,T-t\}}
        \right).
    \end{align*}
    If \(t>T-N\), then both minima are equal to \(T-t\), so the corresponding
    summand is zero. 
    If \(t\le T-N\), then
        $\min\{N,T-t\}=N$
        and
        $
        \min\{N-1,T-t\}=N-1.$
    Therefore, we obtain
    \begin{align*}
        S_{N+1}-S_N
        &=
        \sum_{t=1}^{T-N}
        \left(
            (1+\epsilon)^{N}
            -
            (1+\epsilon)^{N-1}
        \right)
        \\
        &=
        (T-N)
        \left(
            (1+\epsilon)^{N}
            -
            (1+\epsilon)^{N-1}
        \right)
        \\
        &=
        (T-N)
        (1+\epsilon)^{N-1}
        \epsilon
        > 0
        .
    \end{align*}
    Thus \(S_N\), and therefore \(B_N=\xi\epsilon S_N\), is strictly increasing in \(N\).
    Consequently, the Hoeffding concentration penalty
      $  B_N
        \sqrt{
            \frac{2\log(1/\alpha)}{G}
        }$
    is also strictly increasing in \(N\). 
    This completes the proof of Theorem~\ref{thm:n_step_policy_improvement}.
\end{proof}

\section{Implementation Details}
\label{app:implementation details}

\paragraph{Framework.}
We implemented \AlgName\, using the verl~\citep{sheng2025hybridflow} framework, which is a widely used training pipeline for LLM reinforcement learning.
The experiments with baseline algorithms follow the official implementation included in the verl framework, and we additionally implemented the loss function for \AlgName.
Since other components of the training pipeline, including data generation, reward computation, and advantage estimation, remain unchanged, our experiments are easily reproducible by replacing the loss function.
All models are trained on 8 H200 GPUs. Each training run requires approximately one day for the 1.7B model and two days for the 8B model.

\paragraph{Hyperparameters.}
Unless otherwise specified, all hyperparameters used in the experiments follow the default settings of verl.
Hyperparameters that are common across all algorithms are summarized in Table~\ref{tab:common hyperparams}.
For GRPO, we adopt the standard clip-higher configuration ($\epsilon_\mathrm{low}=0.2, \epsilon_\mathrm{high}=0.28$).
For DPPO and \AlgName, we use a total variation–based token mask with a threshold of $\delta=0.2$, as suggested in \citep{qi2026rethinking}.
No auxiliary KL or entropy regularization is applied in any of the algorithms.

For NFPO, we use forward trace clip range $[0.8, 1.4]$ and likelihood ratio clip threshold $\beta=3.0$.
The trace horizon is set to $N=4$ for the 1.7B experiments, and $N=8$ for the 8B experiments.

For the ablation study on token mask (the right plot of Figure~\ref{fig:ablation misc}), we use the standard hyperparameters for each token mask: PPO clip with the clip-higher configuration ($\epsilon_\mathrm{low}=0.2, \epsilon_\mathrm{high}=0.28$), IcePop's token mask with ratio threshold $\beta=3.0$, KL divergence mask with threshold $\delta=0.05$.

\begin{table*}[h]
    \centering
    \small
    \caption{Common hyperparameters shared across all algorithms.}
    \begin{tabular}{llc}
    \toprule
    \textbf{Category} & \textbf{Hyperparameter} & \textbf{Value} \\
    \midrule
    
    \multirow{3}{*}{Optimization}
    & Learning Rate &   $1\times10^{-6}$ \\
    & Optimizer & AdamW ($\beta_1=0.9, \beta_2=0.999$)  \\
    & Weight Decay & $0.01$ \\
    
    \midrule
    \multirow{6}{*}{Data}
    & Prompt per Batch & 128 \\
    & Rollouts per Prompt ($G$) & 8 \\
    & Mini-batch Size & 256 (across all DP ranks) \\
    & Rollout Temperature & $1.0$ for training, $0.7$ for evaluation \\
    & Max Response Tokens & 8,000 (MATH), 16,000 (DAPO-Math-17k) \\
    & Training Steps & 500 (MATH), 300 (DAPO-Math-17k) \\
    
    \midrule
    \multirow{3}{*}{Algorithm}
    & KL Loss & No \\
    & Entropy Loss & No \\
    & Log probability recompute & No \\
    
    \bottomrule
    \end{tabular}
    \label{tab:common hyperparams}
\end{table*}

\paragraph{Details on the Token MDP in Section~\ref{sec:alg}.}

We consider a simple token-level MDP defined over a fixed vocabulary $\mathcal{V} = {a,b,c}$ with a finite horizon $T=7$.
A trajectory corresponds to a sequence $y = (y_1, \dots, y_T) \in \mathcal{V}^T$, which can be interpreted as a string generated autoregressively.
At each timestep $t$, the state is given by the prefix $s_t = (y_1, \dots, y_{t-1})$, and the policy outputs a distribution over the next token.

The reward is defined as a sparse binary signal based on whether the generated sequence contains a predefined target pattern as a subsequence. 
Concretely, given a target string $y^\star = \texttt{abcabc}$, the reward function is
\[
R(y) := \mathbf{1}\big\{ y^\star \text{ is a subsequence of } y \big\}
\]
i.e., the reward is $1$ if the sequence $y$ contains $y^\star$ in order (not necessarily contiguously), and $0$ otherwise.
This reward structure introduces a long-horizon dependency, as achieving a nonzero reward requires consistently selecting tokens that match the target pattern throughout the sequence.

Both the rollout policy $\mu$ and the target policy $\pi$ are instantiated from the same parametric family, which we refer to as the \emph{target-following policy}.
Given a prefix $s$, let $k(s)$ denote the length of the longest prefix of $y^\star$ that appears as a subsequence in $s$.
The policy then assigns probability $\alpha$ to the next required target token $y^\star_{k(s)+1}$, and distributes the remaining mass uniformly over the other tokens:
\[
\pi_\alpha(v|s) = \begin{cases}
    \alpha & \text{if } v=y^\star_{k(s)+1} \\
    \frac{1-\alpha}{|\mathcal{V}|-1} & \text{otherwise.}
\end{cases}
\]
When the full target has already been matched ($k(s) = |y^\star|$), the policy becomes uniform over the vocabulary.

We construct $\mu$ and $\pi$ by choosing different values of $\alpha$, with $\alpha_\mu < \alpha_\pi$.
Intuitively, both policies are biased toward completing the target pattern, but $\pi$ is more strongly aligned with the optimal behavior.
In our experiments, we use $\alpha_\mu = 0.5$ and $\alpha_\pi = 0.8$.

\section{Detailed Token Analysis}
\label{app:token_full}

\raggedbottom

This appendix expands the analysis in Section~\ref{sec:analysis} along three directions.
Section~\ref{app:token_categories} gives the rule-based definitions of the token categories used throughout the analysis.
Section~\ref{app:token_full_aggregate} reports the full per-category correction strength of token-level weights under GRPO and \AlgName{} that is summarized in the main section.
Section~\ref{app:case_studies} presents representative rollout cases with token-level tables of the likelihood ratio $\rho_t$, the raw forward trace $\rawtrace$, the clipped trace $\cliptrace_{t+1}$, and the token advantage; orange highlights mark spans whose forward trace is clipped at either the upper or lower bound, and the highlight is meant to identify a local span of clipped tokens.

\subsection{Token Categories}
\label{app:token_categories}

Every vocabulary token is assigned to exactly one of eight categories by a rule-based classifier on the token string.

\begin{itemize}[leftmargin=1.2em,itemsep=1pt,topsep=2pt]
\item \textbf{Reasoning connectives.} Logical and discourse connectors that link reasoning steps within or across sentences, e.g.\ \textit{therefore, since, because, however, thus, also, first, next, finally, wait, actually}.
\item \textbf{Reasoning actions.} Verbs, modals, and pronouns that introduce a reasoning step or operation, e.g.\ \textit{we, let, can, need, find, solve, use, consider, assume, check, verify, suppose, given, define}.
\item \textbf{Function words.} Closed-class English function words: articles, prepositions, pronouns, and auxiliaries (\textit{the, of, to, is, and, that, which}).
\item \textbf{Math content nouns.} Mathematical content nouns and adjectives such as \textit{equation, function, variable, root, factor, triangle, integer, prime, solution}.
\item \textbf{Math variables.} Single-letter variable tokens such as \textit{x, y, n, k, b}, excluding single-letter English words \textit{a}, \textit{A}, and \textit{I}.
\item \textbf{Numeric.} Digits, decimals, arithmetic and comparison operators, LaTeX math commands such as \texttt{\textbackslash frac}, \texttt{\textbackslash sqrt}, \texttt{\textbackslash cdot}, and number words (\textit{one, two, hundred, half, quarter}).
\item \textbf{Formatting.} Whitespace, BPE prefix-only tokens, newlines, brackets, punctuation runs, markdown markers, special tokens, and LaTeX formatting commands such as \texttt{\textbackslash boxed}, \texttt{\textbackslash begin}, \texttt{\textbackslash text}, \texttt{\textbackslash left}, \texttt{\textbackslash right}.
\item \textbf{Other.} Residual BPE pieces and rare strings that do not match any of the above.
\end{itemize}

\begin{table}[h!]
\centering
\small
\setlength{\tabcolsep}{6pt}
\renewcommand{\arraystretch}{1.3}
\caption{
    Per-category correction strength of token-level weights.
}
\label{tab:appx_token_full}
\begin{tabular}{l c c}
\toprule
 & \multicolumn{2}{c}{Correction strength} \\
\cmidrule(lr){2-3}
Category & $\rho_t$ & $\rho_t \cdot \bar{\Gamma}^{(N,i)}_{t+1}$ \\
\midrule
All tokens (baseline)          & $0.02775$ & $0.07218$ \\
\midrule
Reasoning action               & $0.05167$ & $0.11118$ \\
Reasoning connective           & $0.05595$ & $0.10099$ \\
Function word                  & $0.03834$ & $0.09618$ \\
Math content noun              & $0.04701$ & $0.10218$ \\
Other (residual)               & $0.04025$ & $0.09044$ \\
\midrule
Math variable (single letter)  & $0.02387$ & $0.06332$ \\
Formatting                     & $0.02214$ & $0.06603$ \\
Numeric                        & $0.01789$ & $0.05477$ \\
\bottomrule
\end{tabular}
\end{table}

\subsection{Correction Strength of Per-Token Weights by Category}
\label{app:token_full_aggregate}

This subsection reports the full per-category correction strength.
The pool consists of approximately $3.4\mathrm{M}$ valid response tokens sampled during \AlgName{} training.
For each category, Table~\ref{tab:appx_token_full} reports the correction strength of $\rho_t$  and $\rho_t \cdot \bar{\Gamma}^{(N,i)}_{t+1}$ for \AlgName{}.

Two patterns are consistent across categories.
First, $\rho_t \cdot \bar{\Gamma}^{(N,i)}_{t+1}$ achieves substantially larger correction strength than $\rho_t$ in every category, confirming that the forward trace produces more active corrections than the likelihood ratio alone.
Second, reasoning-related categories (reasoning action, reasoning connective) and content-bearing categories (function word, math content noun) carry substantially larger correction strength than residual categories such as numeric, formatting, and single-letter math variables under both $\rho_t$ and $\rho_t \cdot \bar{\Gamma}^{(N,i)}_{t+1}$.
This pattern supports the claim that the forward trace amplifies corrections on the tokens that drive reasoning behavior.

\newcommand{\TokenTableAOne}{
\begingroup
\par\smallskip
\scriptsize
\setlength{\tabcolsep}{3.5pt}
\begin{xltabular}{\linewidth}{@{}r Y c c c c@{}}
\toprule
idx & token & $\rho_t$ & $\rawtrace$ & $\cliptrace_{t+1}$ & $\hat{A}_t$ \\
\midrule
\endfirsthead
\toprule
idx & token & $\rho_t$ & $\rawtrace$ & $\cliptrace_{t+1}$ & $\hat{A}_t$ \\
\midrule
\endhead
\midrule
\endfoot
\bottomrule
\endlastfoot
\cliprow
17 & \texttt{we}       & 1.000 & 2.67 & 1.400 & $+0.625$ \\
\cliprow
18 & \texttt{first}    & 0.976 & 2.74 & 1.400 & $+0.625$ \\
\cliprow
19 & \texttt{need}     & 1.000 & 2.74 & 1.400 & $+0.625$ \\
\cliprow
20 & \texttt{to}       & 1.000 & 2.77 & 1.400 & $+0.625$ \\
\cliprow
21 & \texttt{identify} & 1.000 & 4.77 & 1.400 & $+0.625$ \\
\cliprow
22 & \texttt{all}      & 2.737 & 1.74 & 1.400 & $+0.625$ \\
\cliprow
23 & \texttt{two}      & 1.000 & 1.74 & 1.400 & $+0.625$ \\
\cliprow
24 & \texttt{-digit}   & 1.000 & 1.74 & 1.400 & $+0.625$ \\
\cliprow
25 & \texttt{prime}    & 1.000 & 1.74 & 1.400 & $+0.625$ \\
\cliprow
26 & \texttt{numbers}  & 1.000 & 1.74 & 1.400 & $+0.625$ \\
\cliprow
27 & \texttt{that}     & 1.011 & 1.73 & 1.400 & $+0.625$ \\
28 & \texttt{remain}   & 1.725 & 1.00 & 1.000 & $+0.625$ \\
\end{xltabular}
\par\smallskip
\endgroup
}

\newcommand{\TokenTableATwo}{
\begingroup
\par\smallskip
\scriptsize
\setlength{\tabcolsep}{3.5pt}
\begin{xltabular}{\linewidth}{@{}r Y c c c c@{}}
\toprule
idx & token & $\rho_t$ & $\rawtrace$ & $\cliptrace_{t+1}$ & $\hat{A}_t$ \\
\midrule
\endfirsthead
\toprule
idx & token & $\rho_t$ & $\rawtrace$ & $\cliptrace_{t+1}$ & $\hat{A}_t$ \\
\midrule
\endhead
\midrule
\endfoot
\bottomrule
\endlastfoot
2860 & \tokdot{}\texttt{is}       & 0.965 & 0.97 & 0.973 & $+0.875$ \\
2861 & \tokdot{}\texttt{counted}  & 1.006 & 0.97 & 0.965 & $+0.875$ \\
2862 & \tokdot{}                  & 1.000 & 0.96 & 0.965 & $+0.875$ \\
\cliprow
2863 & \texttt{4}                 & 0.999 & 1.58 & 1.400 & $+0.875$ \\
\cliprow
2864 & \tokdot{}\texttt{times}    & 1.000 & 1.58 & 1.400 & $+0.875$ \\
\cliprow
2865 & \tokdot{}\texttt{in}       & 1.000 & 1.59 & 1.400 & $+0.875$ \\
\cliprow
2866 & \tokdot{}\texttt{total}    & 1.002 & 1.82 & 1.400 & $+0.875$ \\
\cliprow
2867 & \texttt{.\textbackslash n\textbackslash n} & 0.965 & 1.88 & 1.400 & $+0.875$ \\
\cliprow
2868 & \textbf{\texttt{Wait}}     & \textbf{0.998} & \textbf{2.94} & \textbf{1.400} & \textbf{$+0.875$} \\
\cliprow
2869 & \texttt{,}                 & 1.000 & 2.95 & 1.400 & $+0.875$ \\
\cliprow
2870 & \tokdot{}\texttt{let}      & 1.637 & 1.80 & 1.400 & $+0.875$ \\
\cliprow
2871 & \texttt{'s}                & 1.000 & 1.80 & 1.400 & $+0.875$ \\
\cliprow
2872 & \tokdot{}\texttt{think}    & 1.006 & 1.79 & 1.400 & $+0.875$ \\
\cliprow
2873 & \tokdot{}\texttt{differently} & 1.147 & 1.56 & 1.400 & $+0.875$ \\
\cliprow
2874 & \texttt{.\textbackslash n\textbackslash n} & 1.000 & 1.56 & 1.400 & $+0.875$ \\
2875 & \texttt{Total}             & 1.559 & 1.00 & 1.003 & $+0.875$ \\
2876 & \tokdot{}\texttt{color}    & 1.003 & 1.00 & 1.000 & $+0.875$ \\
2877 & \texttt{ings}              & 1.000 & 1.00 & 1.000 & $+0.875$ \\
2878 & \texttt{:}                 & 1.000 & 1.00 & 1.000 & $+0.875$ \\
2879 & \tokdot{}\texttt{\textbackslash(} & 1.000 & 1.00 & 1.000 & $+0.875$ \\
2880 & \texttt{2}                 & 1.000 & 1.00 & 1.000 & $+0.875$ \\
2881 & \texttt{\textasciicircum}  & 1.000 & 1.00 & 1.000 & $+0.875$ \\
2882 & \texttt{6}                 & 1.000 & 1.04 & 1.044 & $+0.875$ \\
2883 & \tokdot{}\texttt{=}        & 1.000 & 1.04 & 1.044 & $+0.875$ \\
2884 & \tokdot{}                  & 1.000 & 1.04 & 1.038 & $+0.875$ \\
\end{xltabular}
\par\smallskip
\endgroup
}

\newcommand{\TokenTableBOne}{
\begingroup
\par\smallskip
\scriptsize
\setlength{\tabcolsep}{3.5pt}
\begin{xltabular}{\linewidth}{@{}r Y c c c c@{}}
\toprule
idx & token & $\rho_t$ & $\rawtrace$ & $\cliptrace_{t+1}$ & $\hat{A}_t$ \\
\midrule
\endfirsthead
\toprule
idx & token & $\rho_t$ & $\rawtrace$ & $\cliptrace_{t+1}$ & $\hat{A}_t$ \\
\midrule
\endhead
\midrule
\endfoot
\bottomrule
\endlastfoot
\cliprow
636 & \texttt{Since}   & 1.000 & 2.66 & 1.400 & $+0.875$ \\
\cliprow
637 & \texttt{the}     & 1.000 & 2.66 & 1.400 & $+0.875$ \\
\cliprow
638 & \texttt{pattern} & 1.015 & 2.62 & 1.400 & $+0.875$ \\
\cliprow
639 & \texttt{repeats} & 1.000 & 2.62 & 1.400 & $+0.875$ \\
\cliprow
640 & \texttt{every}   & 1.000 & 2.62 & 1.400 & $+0.875$ \\
\cliprow
641 & \tokdot{}        & 1.000 & 2.62 & 1.400 & $+0.875$ \\
\cliprow
642 & \texttt{3}       & 1.000 & 2.62 & 1.400 & $+0.875$ \\
643 & \texttt{ex}      & 2.618 & 1.00 & 0.999 & $+0.875$ \\
644 & \texttt{ponents} & 1.000 & 0.97 & 0.967 & $+0.875$ \\
\end{xltabular}
\par\smallskip
\endgroup
}

\newcommand{\TokenTableBTwo}{
\begingroup
\par\smallskip
\scriptsize
\setlength{\tabcolsep}{3.5pt}
\begin{xltabular}{\linewidth}{@{}r Y c c c c@{}}
\toprule
idx & token & $\rho_t$ & $\rawtrace$ & $\cliptrace_{t+1}$ & $\hat{A}_t$ \\
\midrule
\endfirsthead
\toprule
idx & token & $\rho_t$ & $\rawtrace$ & $\cliptrace_{t+1}$ & $\hat{A}_t$ \\
\midrule
\endhead
\midrule
\endfoot
\bottomrule
\endlastfoot
\cliprow
1439 & \texttt{)}       & 1.000 & 2.95 & 1.400 & $+0.375$ \\
\cliprow
1440 & \texttt{(}       & 1.000 & 4.14 & 1.400 & $+0.375$ \\
\cliprow
1441 & \texttt{Invalid} & 1.000 & 4.14 & 1.400 & $+0.375$ \\
\cliprow
1442 & \texttt{)\textbackslash n\textbackslash n} & 1.000 & 4.14 & 1.400 & $+0.375$ \\
\cliprow
1443 & \texttt{So}      & 1.000 & 4.14 & 1.400 & $+0.375$ \\
\cliprow
1444 & \texttt{,}       & 1.000 & 4.14 & 1.400 & $+0.375$ \\
\cliprow
1445 & \texttt{the}     & 1.000 & 4.14 & 1.400 & $+0.375$ \\
\cliprow
1446 & \texttt{valid}   & 2.954 & 1.40 & 1.400 & $+0.375$ \\
1447 & \texttt{pairs}   & 1.404 & 1.00 & 1.000 & $+0.375$ \\
1448 & \texttt{are}     & 1.000 & 1.00 & 1.000 & $+0.375$ \\
1449 & \texttt{:}       & 1.000 & 1.00 & 1.000 & $+0.375$ \\
\end{xltabular}
\par\smallskip
\endgroup
}

\newcommand{\TokenTableDOne}{
\begingroup
\par\smallskip
\scriptsize
\setlength{\tabcolsep}{3.5pt}
\begin{xltabular}{\linewidth}{@{}r Y c c c c@{}}
\toprule
idx & token & $\rho_t$ & $\rawtrace$ & $\cliptrace_{t+1}$ & $\hat{A}_t$ \\
\midrule
\endfirsthead
\toprule
idx & token & $\rho_t$ & $\rawtrace$ & $\cliptrace_{t+1}$ & $\hat{A}_t$ \\
\midrule
\endhead
\midrule
\endfoot
\bottomrule
\endlastfoot
867 & \texttt{\{}              & 1.000 & 0.90 & 0.900 & $-0.500$ \\
\cliprow
868 & \texttt{7}                & 1.000 & 0.70 & 0.800 & $-0.500$ \\
\cliprow
869 & \texttt{\}}              & 1.000 & 0.70 & 0.800 & $-0.500$ \\
\cliprow
870 & \tokdot{}\texttt{\textbackslash} & 1.000 & 0.70 & 0.800 & $-0.500$ \\
\cliprow
871 & \texttt{]\textbackslash n\textbackslash n} & 1.000 & 0.70 & 0.800 & $-0.500$ \\
\cliprow
872 & \textbf{\texttt{However}} & \textbf{0.998} & \textbf{0.56} & \textbf{0.800} & \textbf{$-0.500$} \\
\cliprow
873 & \texttt{,}                & 1.000 & 0.56 & 0.800 & $-0.500$ \\
\cliprow
874 & \tokdot{}\texttt{this}    & 0.901 & 0.62 & 0.800 & $-0.500$ \\
\cliprow
875 & \tokdot{}\texttt{does}    & 0.782 & 0.80 & 0.800 & $-0.500$ \\
\cliprow
876 & \tokdot{}\texttt{not}     & 1.000 & 0.80 & 0.800 & $-0.500$ \\
\cliprow
877 & \tokdot{}\texttt{fit}     & 1.000 & 0.80 & 0.800 & $-0.500$ \\
\cliprow
878 & \tokdot{}\texttt{the}     & 0.999 & 0.80 & 0.800 & $-0.500$ \\
879 & \tokdot{}\texttt{context} & 0.798 & 0.98 & 0.984 & $-0.500$ \\
880 & \tokdot{}\texttt{of}      & 0.999 & 1.05 & 1.047 & $-0.500$ \\
881 & \tokdot{}\texttt{the}     & 1.000 & 1.05 & 1.047 & $-0.500$ \\
882 & \tokdot{}\texttt{problem} & 1.000 & 1.05 & 1.047 & $-0.500$ \\
883 & \texttt{.}                & 0.999 & 1.08 & 1.083 & $-0.500$ \\
884 & \tokdot{}\texttt{Let}     & 1.000 & 1.05 & 1.052 & $-0.500$ \\
885 & \texttt{'s}               & 1.000 & 1.05 & 1.052 & $-0.500$ \\
886 & \tokdot{}\texttt{re}      & 0.986 & 1.05 & 1.052 & $-0.500$ \\
887 & \texttt{-e}               & 1.063 & 0.94 & 0.942 & $-0.500$ \\
888 & \texttt{valuate}          & 1.000 & 0.94 & 0.942 & $-0.500$ \\
\end{xltabular}
\par\smallskip
\endgroup
}

\subsection{Case study}
\label{app:case_studies}

This subsection presents representative cases from rollouts during \AlgName{} training.
The first four cases show situations where the forward trace amplifies the update, and the model is rewarded more strongly than the local ratio alone would suggest. These spans typically initiate or transition between reasoning steps, and the forward trace pushes the policy further in their direction.
The last case shows the opposite situation, where the forward trace dampens the update on a failed trajectory. Even when the final answer is wrong, spans that contain valid reasoning attempts receive a softer penalty, so that the update does not punish the reasoning move itself.

\begin{tracecasebox}{Case 1: Goal definition}
\textbf{Prompt.}
Palindromic primes are two-digit prime numbers whose reversed digits also form a prime.
Find the sum of all palindromic primes less than $50$.

\begin{traceexcerpt}
\textbf{Response excerpt.}

To find the sum of all palindromic primes less than $50$,
\clipToken{we} \clipToken{first} \clipToken{need} \clipToken{to}
\clipToken{identify} \clipToken{all} \clipToken{two}\clipToken{-digit}
\clipToken{prime} \clipToken{numbers} \clipToken{that}
remain prime when their digits are reversed. \ldots
\end{traceexcerpt}

\TokenTableAOne

The highlighted span is an early planning phrase, not part of the final answer.
Several tokens have local ratios close to one, but their forward traces are large enough to hit the upper clipping bound because the continuation opened by this subgoal becomes substantially more likely under the updated policy.
This shows that the forward trace captures a span-level signal that is mostly invisible to the local token ratio alone.

\end{tracecasebox}

\begin{tracecasebox}{Case 2: Self-correction to an alternative approach}

\textbf{Prompt.}
Each face of a cube is painted either red or blue, each with probability $1/2$.
The color of each face is determined independently.
What is the probability that the painted cube can be placed on a horizontal surface so that the four vertical faces are all the same color?

\begin{traceexcerpt}
\textbf{Response excerpt.}

we can find a smarter way.

\medskip
\noindent Considering the Complementary Event:

Let's think about the complementary event: it is counted
\clipToken{4} \clipToken{times} \clipToken{in} \clipToken{total}\clipToken{.}

\clipToken{Wait}\clipToken{,}
\clipToken{let}\clipToken{'s} \clipToken{think} \clipToken{differently}\clipToken{.}

Total colorings: \(2^6=\cdots\).
\end{traceexcerpt}

\TokenTableATwo

This case shows a clear self-correction in the reasoning trajectory.
The model first attempts a complementary-event argument, notices that the counting structure is becoming awkward, and then switches strategy with the phrase ``Wait, let's think differently.''
The focus token \texttt{Wait} has an almost unchanged local ratio ($\rho_t=0.998$), but its raw forward trace rises to $\rawtrace=2.94$ and is clipped to $\cliptrace_{t+1}=1.400$.
This is a case where the local ratio fails to capture but the forward trace captures.

\end{tracecasebox}

\begin{tracecasebox}{Case 3: Causal connector}

\textbf{Prompt.}
Find the remainder when $1+2+2^2+\cdots+2^{100}$ is divided by $7$.

\begin{traceexcerpt}
\textbf{Response excerpt.}

\ldots{} the exponent modulo $3$.
\clipToken{Since} \clipToken{the} \clipToken{pattern}
\clipToken{repeats} \clipToken{every} \clipToken{3}
exponents, we can find $2^{101}\bmod 7$. \ldots
\end{traceexcerpt}

\TokenTableBOne

The highlighted span is a causal transition into the periodicity argument.
The tokens themselves do not perform the computation, but together they mark the step where a modular pattern becomes the final answer.
The forward trace assigns high weight to this connector span, even though the local ratios alone leave it nearly unchanged.

\end{tracecasebox}

\begin{tracecasebox}{Case 4: Consolidation connector}

\textbf{Prompt.}
A standard six-sided die is rolled twice.
What is the probability that the difference between the two rolls is $2$?

\begin{traceexcerpt}
\textbf{Response excerpt.}

\ldots{} $y=4,8$)
\clipToken{(}\clipToken{Invalid}\clipToken{)}
\clipToken{So}\clipToken{,} \clipToken{the} \clipToken{valid}
pairs are: $(1,3), (2,4), \ldots$
\end{traceexcerpt}

\TokenTableBTwo

The highlighted span is a consolidation connector that moves from rejecting invalid outcomes to listing the valid pairs.
The forward trace weights this consolidation step as a whole rather than to any single token.

\end{tracecasebox}

\begin{tracecasebox}{Case 5: Self-correction in a failed trajectory}

\textbf{Prompt.}
Either increasing the radius or the height of a cylinder by six inches will result in the same volume.
The original height of the cylinder is two inches.
What is the original radius in inches?

\begin{traceexcerpt}
\textbf{Response excerpt.}

\ldots{} $2\sqrt{7}\,]$
\clipToken{However}\clipToken{,}
\clipToken{this} \clipToken{does} \clipToken{not}
\clipToken{fit} \clipToken{the}
context of the problem. Let's re-evaluate \ldots
\end{traceexcerpt}

\TokenTableDOne

This case is the symmetric counterpart of Case~2.
The trajectory ultimately fails (reward $0$) and the global advantage is therefore negative ($\hat{A}_t = -0.500$).
A purely local surrogate would penalize every token in the trajectory by the same advantage scale.
The forward trace, however, drops to $\rawtrace = 0.56$ on the token \texttt{However} and stays low across the whole self-correction phrase ``However, this does not fit the context $\ldots$ Let's re-evaluate''; the clipped trace at the lower bound $0.8$ still reduces the effective gradient on this span by $20\%$ relative to the local-only update.

\end{tracecasebox}

\section{Detailed Experimental Results}
\label{app:detailed experimental results}

This section reports the detailed learning dynamics of the experiments in Section~\ref{sec:main_results}, along with per-benchmark performance comparison of the ablation experiments in Section~\ref{sec:ablation}.

\begin{itemize}
    \item Figures~\ref{fig:learning_dynamics_8b} and~\ref{fig:learning_dynamics_1.7b} compare the training dynamics of GRPO, DPPO, and \AlgName{} on Qwen3-8B-Base and Qwen3-1.7B-Base, respectively, showing reward, mean response length, KL divergence, and per-benchmark accuracy throughout training.
    
    \item Table~\ref{tab:ablation_n} reports the effect of the trace horizon $N$, where $N{=}4$ achieves the best average score and all configurations with $N \geq 2$ outperform shorter horizons.

    \item Table~\ref{tab:clip_range} reports the effect of forward trace clip ranges, showing that performance is robust to the choice of clip range.
    
    \item Table~\ref{tab:design choice} shows the effect of recomputing the log probability $\mu(y_t |s_t)$ at each gradient step, which causes a slight performance drop and additional computational cost. Additionally, it reports the effect of clipping $\rho_t$, showing that removing this clipping yields nearly identical performance.

    \item Finally, Table~\ref{tab:trust_region_compat} compares different token mask strategies for \AlgName{}, including the PPO clip, the token masks based on binary total variation (TV) and KL divergence, and IcePop's token mask.
\end{itemize}

\begin{figure}[H]
    \centering
    \includegraphics[width=\linewidth]{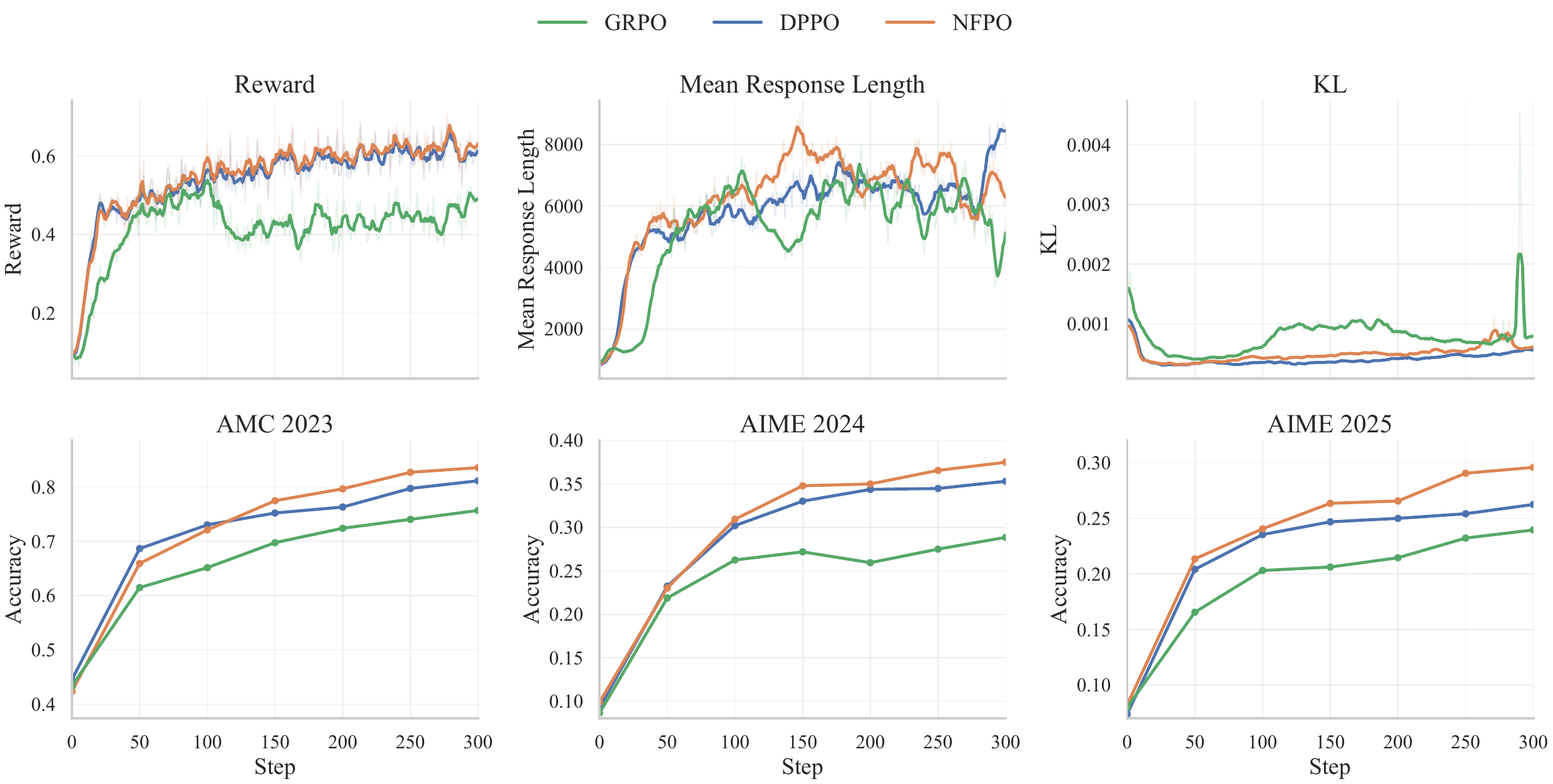}
    \caption{
    \small{Learning dynamics on Qwen3-8B-Base. Top row: training reward, mean response length, and KL divergence. Bottom row: per-benchmark accuracy.}}
    \label{fig:learning_dynamics_8b}
\end{figure}

\begin{figure}[H]
    \centering
    \includegraphics[width=\linewidth]{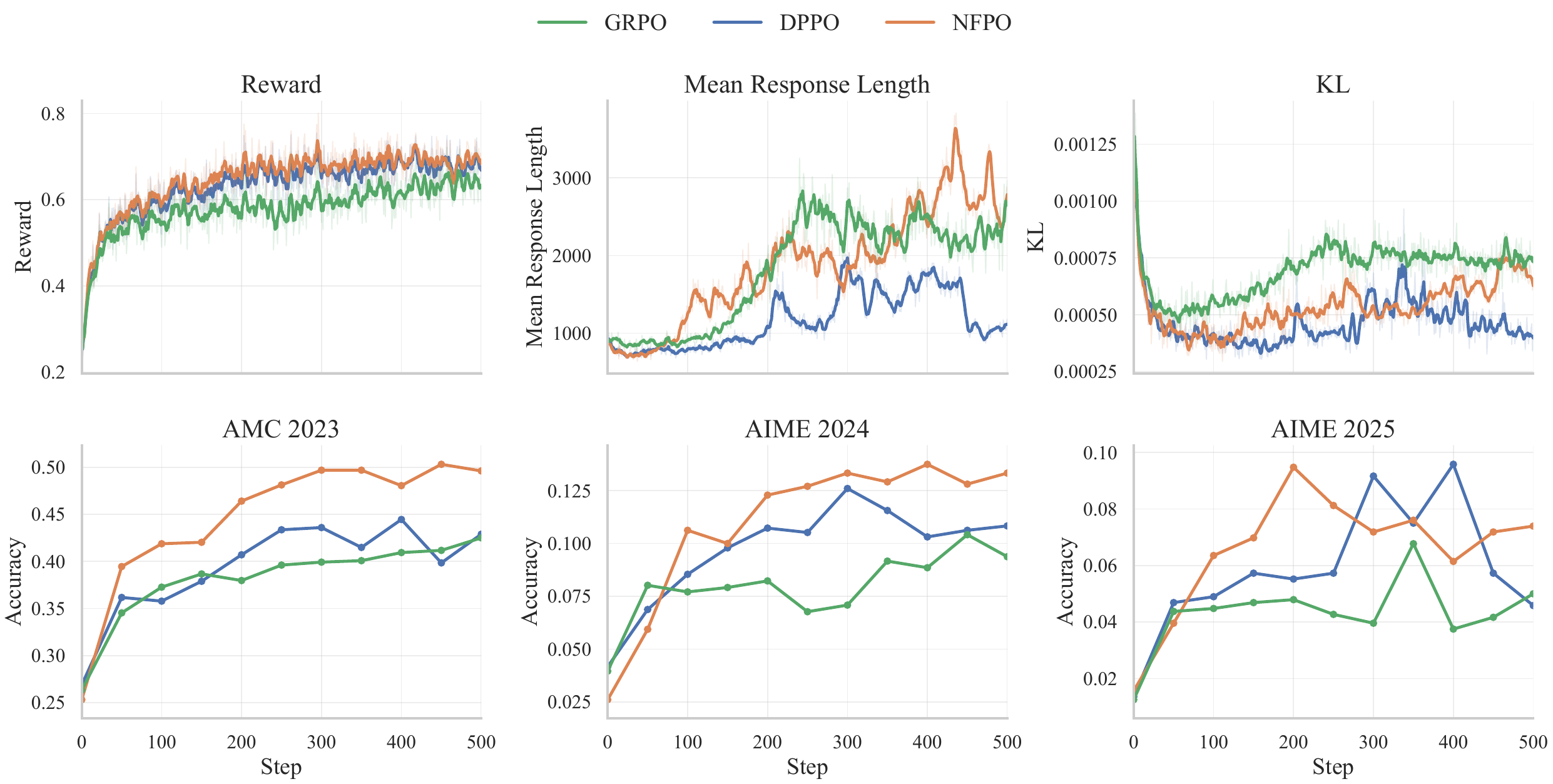}
    \caption{
    \small{Learning dynamics on Qwen3-1.7B-Base. Top row: training reward, mean response length, and KL divergence. Bottom row: per-benchmark accuracy.}}
    \label{fig:learning_dynamics_1.7b}
\end{figure}

\begin{table}[H]
    \centering
    \caption{
    \small{Performance comparison across different values of $N$ on mathematical reasoning benchmarks.}}
    
    \setlength{\tabcolsep}{4pt}
    \small 
    
    \begin{tabular}{@{}c | c c c c c c c | c@{}}
        \toprule
        \textbf{} & AIME24 & AIME25 & AIME26 & AMC23 & MATH 500 & Minerva & Oly. & Avg. \\
        \midrule
        $N=1$ (DPPO) & $0.106$ & $0.066$ & $0.056$ & $0.424$ & $0.715$ & $0.246$ & $0.341$ & $0.279$ \\
        $N=2$ & $0.114$ & $0.082$ & $0.053$ & $0.431$ & $0.718$ & $0.246$ & $0.344$ & $0.284$ \\
        $N=4$ & $0.133$ & $0.074$ & $0.081$ & $0.496$ & $0.745$ & $0.262$ & $0.360$ & $0.307$ \\
        $N=8$ & $0.124$ & $0.114$ & $0.050$ & $0.473$ & $0.748$ & $0.244$ & $0.355$ & $0.301$ \\
        $N=16$ & $0.105$ & $0.078$ & $0.079$ & $0.433$ & $0.749$ & $0.254$ & $0.370$ & $0.296$ \\
        \bottomrule
    \end{tabular}
    \label{tab:ablation_n}
\end{table}

\begin{table}[H]
    \centering
    \caption{
    \small{Ablation study on the forward trace clip ranges for \AlgName{}.}}
    
    \setlength{\tabcolsep}{4pt}
    \small 
    
    \begin{tabular}{@{}c | c c c c c c c | c@{}}
        \toprule
        \textbf{Clip range} & AIME24 & AIME25 & AIME26 & AMC23 & MATH 500 & Minerva & Oly. & Avg. \\
        \midrule
        $[0.6, 1.2]$ & $0.117$ & $0.067$ & $0.050$ & $0.504$ & $0.737$ & $0.270$ & $0.359$ & $0.300$ \\
        $[0.8, 1.2]$ & $0.123$ & $0.071$ & $0.079$ & $0.493$ & $0.731$ & $0.247$ & $0.358$ & $0.300$ \\
        $[0.8, 1.4]$ & $0.133$ & $0.074$ & $0.081$ & $0.496$ & $0.745$ & $0.262$ & $0.360$ & $0.307$ \\
        \bottomrule
    \end{tabular}
    \label{tab:clip_range}
\end{table}

\begin{table}[H]
    \centering
    \caption{
    \small{Ablation study on the log probability recompute and likelihood ratio ($\rho_t$) clip.}}
    
    \setlength{\tabcolsep}{4pt}
    \small 
    
    \begin{tabular}{@{}r | c c c c c c c | c@{}}
        \toprule
        \textbf{} & AIME24 & AIME25 & AIME26 & AMC23 & MATH 500 & Minerva & Oly. & Avg. \\
        \midrule
        NFPO & $0.133$ & $0.074$ & $0.081$ & $0.496$ & $0.745$ & $0.262$ & $0.360$ & $0.307$ \\
        w/ recompute  & $0.132$ & $0.081$ & $0.049$ & $0.440$ & $0.734$ & $0.257$ & $0.355$ & $0.293$ \\
        w/o $\rho_t$ clip & $0.135$ & $0.103$ & $0.076$ & $0.501$ & $0.739$ & $0.234$ & $0.361$ & $0.307$ \\
        \bottomrule
    \end{tabular}
    \label{tab:design choice}
\end{table}

\begin{table}[H]
    \centering
    \caption{
    \small{Ablation study on \AlgName{}'s token mask.}}
    
    \setlength{\tabcolsep}{4pt}
    \small 
    
    \begin{tabular}{@{}l | c c c c c c c | c@{}}
        \toprule
        \textbf{Mask Type} & AIME24 & AIME25 & AIME26 & AMC23 & MATH 500 & Minerva & Oly. & Avg. \\
        \midrule
        TV & $0.133$ & $0.074$ & $0.081$ & $0.496$ & $0.745$ & $0.262$ & $0.360$ & $0.307$ \\
        KL & $0.044$ & $0.067$ & $0.128$ & $0.439$ & $0.718$ & $0.237$ & $0.336$ & $0.281$ \\
        PPO  & $0.096$ & $0.054$ & $0.043$ & $0.424$ & $0.690$ & $0.229$ & $0.323$ & $0.266$ \\
        IcePop & $0.087$ & $0.036$ & $0.033$ & $0.417$ & $0.686$ & $0.229$ & $0.313$ & $0.258$ \\
        \bottomrule
    \end{tabular}
    \label{tab:trust_region_compat}
\end{table}

\end{document}